\documentclass[journal]{IEEEtran}
\usepackage[linesnumbered,ruled,vlined]{algorithm2e}
\usepackage{amsmath}
\usepackage{epsfig,amssymb,subfigure,bm,dsfont}
\usepackage{algpseudocode}
\usepackage{amssymb}
\usepackage{amsfonts}
\usepackage{array}
\usepackage{balance}
\usepackage{bbding}
\usepackage{booktabs}
\usepackage{calc}
\usepackage{caption}
\usepackage{cite}
\usepackage{color}
\usepackage{curves}
\usepackage{epsfig}
\usepackage{epstopdf}
\usepackage{epic}
\usepackage{float}
\usepackage{footnote}
\usepackage{graphicx}
\usepackage{mathrsfs}
\usepackage{mdwlist}
\usepackage{multicol}
\usepackage{ntheorem}
\usepackage{pifont}
\usepackage{psfrag}
\usepackage{url}
\usepackage{subfigure}
\usepackage{stfloats}
\usepackage{stmaryrd}
\usepackage{times}
\usepackage{tipa}
\usepackage[english]{babel}
\begin{document}
\bstctlcite{IEEEexample:BSTcontrol}
\theoremheaderfont{\bfseries\upshape}
\theoremseparator{:}
\newtheorem{proof}{Proof}
\newtheorem{theorem}{Theorem}
\newtheorem{assumption}{Assumption}
\newtheorem{proposition}{Proposition}
\newtheorem{definition}{Definition}
\newtheorem{lemma}{Lemma}
\newtheorem{corollary}{Corollary}
\newtheorem{remark}{Remark}
\newtheorem{construction}{Construction}
\newtheorem{problem}{Problem}
\newtheorem{alg}{Algorithm}[section]
\captionsetup{font={scriptsize}}
\renewcommand{\IEEEQED}{\IEEEQEDclosed}

\newcommand{\supp}{\mathop{\rm supp}}
\newcommand{\sinc}{\mathop{\rm sinc}}
\newcommand{\spann}{\mathop{\rm span}}
\newcommand{\essinf}{\mathop{\rm ess\,inf}}
\newcommand{\esssup}{\mathop{\rm ess\,sup}}
\newcommand{\Lip}{\rm Lip}
\newcommand{\sign}{\mathop{\rm sign}}
\newcommand{\osc}{\mathop{\rm osc}}
\newcommand{\R}{{\mathbb{R}}}
\newcommand{\Z}{{\mathbb{Z}}}
\newcommand{\C}{{\mathbb{C}}}
\newcommand*{\affaddr}[1]{#1} 
\newcommand*{\affmark}[1][*]{\textsuperscript{#1}}
\newcommand*{\email}[1]{\texttt{#1}}

\title{Federated Learning with Differential Privacy: Algorithms and Performance Analysis}
\author{{Kang Wei, Jun Li, Ming Ding, Chuan Ma, Howard H. Yang, Farokhi Farhad,\\ Shi Jin, Tony Q. S. Quek, H. Vincent Poor}
\thanks{Kang~Wei, Jun~Li, and Chuan Ma are with School of Electrical and Optical Engineering, Nanjing University of Science and Technology, Nanjing 210094, China (e-mail:\{wei.kang, jun.li, chuan.ma\}@njust.edu.cn).}
\thanks{Ming~Ding and Farokhi Farhad are with Data61, CSIRO, Sydney, NSW 2015, Australia (e-mail:\{ming.ding, Farhad.Farokhi\}@data61.csiro.au).}
\thanks{Howard~Hao~Yang and Tony~Q.~S.~Quek are with the Information System Technology and Design Pillar, Singapore University of Technology and Design, Singapore (e-mail:\{howard yang, tonyquek\}@sutd.edu.sg).}
\thanks{Shi~Jin is with the National Mobile Communications Research Laboratory, Southeast University, Nanjing 210096, China (e-mail: jinshi@seu.edu.cn).}
\thanks{H.~Vincent~Poor is with the Department of Electrical Engineering, Princeton University, Princeton, NJ 08544 USA (e-mail: poor@princeton.edu).}}
\maketitle
\begin{abstract}
Federated learning (FL), as a manner of distributed machine learning, is capable of significantly preserving clients' private data from being exposed to external adversaries. Nevertheless, private information can still be divulged by analyzing on the differences of uploaded parameters from clients, e.g., weights trained in deep neural networks. In this paper, to effectively prevent information leakage, we propose a novel framework based on the concept of differential privacy (DP), in which artificial noises are added to the parameters at the clients side before aggregating, namely, noising before model aggregation FL (NbAFL). First, we prove that the NbAFL can satisfy DP under distinct protection levels by properly adapting different variances of artificial noises. Then we develop a theoretical convergence bound of the loss function of the trained FL model in the NbAFL. Specifically, the theoretical bound reveals the following three key properties: 1) There is a tradeoff between the convergence performance and privacy protection levels, i.e., a better convergence performance leads to a lower protection level; 2) Given a fixed privacy protection level, increasing the number $N$ of overall clients participating in FL can improve the convergence performance; 3) There is an optimal number of maximum aggregation times (communication rounds) in terms of convergence performance for a given protection level. Furthermore, we propose a $K$-random scheduling strategy, where $K$ ($1<K<N$) clients are randomly selected from the $N$ overall clients to participate in each aggregation. We also develop the corresponding convergence bound of the loss function in this case and the $K$-random scheduling strategy can also retain the above three properties. Moreover, we find that there is an optimal $K$ that achieves the best convergence performance at a fixed privacy level. Evaluations demonstrate that our theoretical results are consistent with simulations, thereby facilitating the designs on various privacy-preserving FL algorithms with different tradeoff requirements on convergence performance and privacy levels.
\end{abstract}


\section{Introduction}
With AlphaGo's glorious success, it is expected that the big data-driven artificial intelligence (AI) will soon be applied in all aspects of our daily life, including medical care, food and agriculture, intelligent transportation systems, etc. At the same time, the rapid proliferations of Internet of Things (IoTs) call for data mining and learning securely and reliably in distributed systems~\cite{8405572, 8673568, 8792179}. When integrating AI in a variety of IoT applications, distributed machine learning (ML) are remarkably effective for many data processing tasks by defining parameterized functions from inputs to outputs as compositions of basic building blocks~\cite{8863729,8373692}. Federated learning (FL), as a recent advance of distributed ML, was proposed, in which data are acquired and processed locally at the clients side, and then the updated ML parameters are transmitted to a central server for aggregating, i.e., averaging on these parameters~\cite{DBLP:journals/corr/McMahanMRA16,DBLP:journals/corr/KonecnyMYRSB16,DBLP:journals/corr/abs-1905-01656}. Typically, clients in FL are distributed devices such as sensors, wearable devices, or mobile phones. The goal of FL is to fit a model generated by an empirical risk minimization (ERM) objective. However, FL also poses several key challenges, such as private information leakage, expensive communication costs between servers and clients, and device variability~\cite{8770530,Yang:2019:FML:3306498.3298981, 2019arXiv190807873L, 8737464, HaoYangFL,8761267}.

Generally, distributed stochastic gradient descent (SGD) is adopted in FL for training ML models. In~\cite{NIPS2011-4247,NIPS2015-5751}, bounds for FL convergence performance were developed based on distributed SGD, with a one-step local update before global aggregations. The work in~\cite{NIPS2017-7117} considered partially global aggregations, where after each local update step, parameter aggregation is performed over a non-empty subset of the clients set. In order to analyze the convergence more effectively, federated proximal (FedProx) was proposed~\cite{DBLP:journals/corr/abs-1812-06127} by adding regularization on each local loss function. The work in~\cite{8664630} obtained the convergence bound of SGD based FL that incorporates non-independent-and-identically-distributed (non-\emph{i.i.d.}) data distributions among clients.

At the same time, with the ever increasing awareness of data security of personal information, privacy preservation has become a worldwide and significant issue, especially for the big data applications and distributed learning systems. One prominent advantage of FL is that it enables local training without personal data exchange between the server and clients, thereby protecting clients' data from being eavesdropped by hidden adversaries. Nevertheless, private information can still be divulged to some extent from adversaries' analyzing on the differences of related parameters trained and uploaded by the clients, e.g., weights trained in neural networks~\cite{Shokri:2015:PDL:2810103.2813687,8737416,ChuanMaFL}.

A natural approach to preventing information leakage is to add artificial noises, known as differentially private (DP) techniques~\cite{TCS-042,Blum:2005:PPS:1065167.1065184}. Existing works on DP based learning algorithms include local~\mbox{DP (LDP)~\cite{Erlingsson:2014:RRA:2660267.2660348, 8731512, 8640266},} DP based distributed SGD~\cite{Abadi:2016:DLD:2976749.2978318, DBLP:journals/corr/abs-1906-09679} and DP meta learning~\cite{2019arXiv190905830L}. In the LDP, each client perturbs its information locally and only sends a randomized version to a server, thereby protecting both the clients and server against private information leakage. The work in~\cite{8731512} proposed solutions to building up a LDP-compliant SGD, which powers a variety of important ML tasks. The work in~\cite{8640266} considered the distribution estimation at the server over uploaded data from clients while providing protections on these data with LDP. The work in~\cite{Abadi:2016:DLD:2976749.2978318} improved the computational efficiency of DP based SGD by tracking detailed information of the privacy loss, and obtained accurate estimates on the overall privacy loss. The work in~\cite{DBLP:journals/corr/abs-1906-09679} proposed novel DP based SGD algorithms and analyzed their performance bounds which are shown to be related to privacy levels and the sizes of datasets. Also, the work in~\cite{2019arXiv190905830L} focused on the class of gradient-based parameter-transfer methods and developed a DP based meta learning algorithm that not only satisfies the privacy requirement but also retains provable learning performance in convex settings.

More specifically, DP based FL approaches are usually devoted to capturing the tradeoff between privacy and convergence performance in the training process. The work in~\cite{DBLP:journals/corr/abs-1712-07557} proposed a FL algorithm with the consideration on preserving clients' privacy. This algorithm can achieve a good training performance at a given privacy level, especially when there is a sufficiently large number of participating clients. The work in~\cite{DBLP:journals/corr/abs-1812-03224} presented an alternative approach that utilizes both DP and secure multiparty computation (SMC) to prevent differential attacks. However, the above two works on DP-based FL design have not taken into account the privacy protection during the parameter uploading stage, i.e., the clients' private information can be potentially intercepted by hidden adversaries when uploading the training results to the server. Moreover, these two works only showed empirical results by simulations, but lacked theoretical analysis on the FL system, such as tradeoff between privacy, convergence performance, and convergence rate. Up to now, the theoretical analysis on convergence behavior of FL with privacy-preserving noise perturbations has not yet been detailed in existing literatures, which will be the major focus of our work in this paper.

In this paper, to effectively prevent information leakage, we propose a novel framework based on the concept of differential privacy (DP), in which each client perturbs its trained parameters locally by purposely adding noises before uploading them to the server for aggregation, namely, noising before model aggregation FL (NbAFL). To the best of authors' knowledge, this is the first piece of work of its kind that theoretically analyzes the convergence property of differentially private FL algorithms. First, we prove that the proposed NbAFL scheme satisfies the requirement of DP in terms of global data under a certain noise perturbation level with Gaussian noises by properly adapting their variances. Then, we develop theoretically a convergence bound of the loss function of the trained FL model in the NbAFL with artificial Gaussian noises. Our developed bound reveals the following three key properties: 1) There is a tradeoff between the convergence performance and privacy protection levels, i.e., a better convergence performance leads to a lower protection level; 2) Increasing the number $N$ of overall clients participating in FL can improve the convergence performance, given a fixed privacy protection level; 3) There is an optimal number of maximum aggregation times in terms of convergence performance for a given protection level. Furthermore, we propose a $K$-random scheduling strategy, where $K$ ($1<K<N$) clients are randomly selected from the $N$ overall clients to participate in each aggregation. We also develop the corresponding convergence bound of the loss function in this case. From our analysis, the $K$-random scheduling strategy can retain the above three properties. Also, we find that there exists an optimal value of $K$ that achieves the best convergence performance at a fixed privacy level. Evaluations demonstrate that our theoretical results are consistent with simulations. Therefore, our analytical results are helpful for the design on privacy-preserving FL architectures with different tradeoff requirements on convergence performance and privacy levels.

The remainder of this paper is organized as follows. In Section~\ref{sec:Preliminaries}, we introduce backgrounds on FL, DP and a conventional DP-based FL algorithm. In Section~\ref{sec:FLDP}, we detail the proposed NbAFL and analyze the privacy performance based on DP. In Section~\ref{sec:Con_NbAFL}, we analyze the convergence bound of NbAFL and reveal the relationship between privacy levels, convergence performance, the number of clients, and the number of global aggregations. In Section~\ref{Sec:K-Random}, we propose the $K$-random scheduling scheme and develop the convergence bound. We show the analytical results and simulations in Section~\ref{Sec:Exm_Res}. We conclude the paper in Section~\ref{Sec:Concl}. A summary of basic concepts and notations is provided in Tab.~\ref{tab:SummaryofMainNotations}.
\begin{table}[h]\caption{Summary of Main Notations}
\centering
\begin{tabular}{ll}
\hline
$\mathcal M$& A randomized mechanism for DP\\
$x,x'$& Adjacent databases\\
$\epsilon, \delta$& The parameters related to DP\\
$\mathcal C_i$& The $i$-th client\\
$\mathcal D_i$& The database held by the owner $\mathcal C_i$\\
$\mathcal D$& The database held by all the clients\\
$|\cdot|$& The cardinality of a set\\
$N$& Total number of all clients\\
$K$& The number of chosen clients ($1<K<N$)\\
$t$& The index of the $t$-th aggregation\\
$T$& The number of aggregation times\\
$\mathbf{w}$& The vector of model parameters\\
$F(\mathbf{w})$& Global loss function\\
$F_{i}(\mathbf{w})$& Local loss function from the $i$-th client\\
$\mu$& A presetting constant of the proximal term\\
$\mathbf{w}^{(t)}_{i}$& Local uploading parameters of the $i$-th client\\
$\mathbf{w}^{(0)}$& Initial parameters of the global model\\
$\mathbf{w}^{(t)}$& Global parameters generated from all local parameters\\
&at the $t$-th aggregation\\
$\mathbf{v}^{(t)}$& Global parameters generated from $K$ clients' parameters\\
&at the $t$-th aggregation\\
$\mathbf{w}^{*}$& True optimal model parameters that minimize $F(\mathbf{w})$\\
$\widetilde{\mathbf{W}}$& The set of all local parameters with pertubation\\
\hline
\end{tabular}
\label{tab:SummaryofMainNotations}
\end{table}

\section{Preliminaries}\label{sec:Preliminaries}
In this section, we will present preliminaries and related background knowledge on FL and DP. Also, we introduce a conventional DP-based FL algorithm that will be discussed in our following analysis as a benchmark.
\begin{figure}[htb]
\centering
\includegraphics[width=3.4in,angle=0]{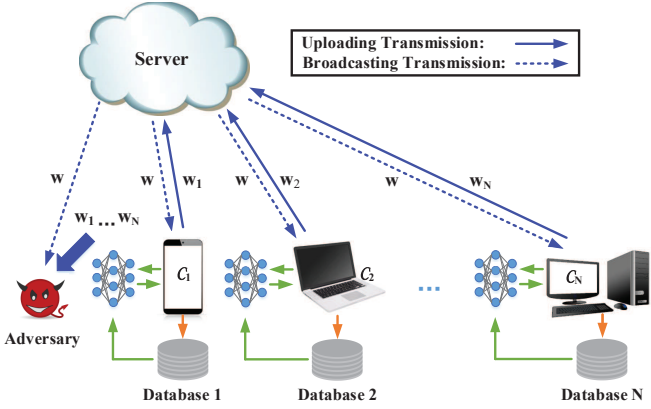}
\caption{A FL training model with hidden adversaries who can eavesdrop trained parameters from both the clients and the server.}\label{fig:Traning model}
\end{figure}
\subsection{Federated Learning}\label{sec:Federated Learning}
Let us consider a general FL system consisting of one server and $N$ clients, as depicted in Fig.~\ref{fig:Traning model}.
Let $\mathcal D_i$ denote the local database held by the client $\mathcal C_i$, where $i\in \{1, 2,\ldots, N\}$. At the server, the goal is to learn a model over data that resides at the $N$ associated clients. An active client, participating in the local training, needs to find a vector $\mathbf{w}$ of an AI model to minimize a certain loss function. Formally, the server aggregates the weights sent from the $N$ clients as
\begin{equation}\label{equ:Aggregating}
\mathbf{w} =\sum_{i=1}^{N}{p_{i}\mathbf{w}_{i}},
\end{equation}
where $\mathbf{w}_{i}$ is the parameter vector trained at the $i$-th client, $\mathbf{w}$ is the parameter vector after aggregating at the server, $N$ is the number of clients, $p_{i} = \frac{\vert \mathcal D_i\vert}{\vert \mathcal D\vert}\geq 0$ with $\sum_{i=1}^{N}{p_{i}}=1$, and $\vert \mathcal D\vert = \sum_{i=1}^{N}{\vert \mathcal D_{i}\vert}$ is the total size of all data samples. Such an optimization problem can be formulated as
\begin{equation}\label{equ:Global objective function}
\mathbf{w}^{*}=\arg\min_{\mathbf{w}}{\sum_{i=1}^{N}{p_{i}F_{i}(\mathbf{w})}},
\end{equation}
where $F_{i}(\cdot)$ is the local loss function of the $i$-th client. Generally, the local loss function $F_{i}(\cdot)$ is given by local empirical risks. The training process of such a FL system usually contains the following four steps:
\begin{basedescript}{\desclabelstyle{\pushlabel}\desclabelwidth{4em}}
\item[$\bullet$ \textbf{Step 1}:]
\emph{Local training: }All active clients locally compute training gradients or parameters and send locally trained ML parameters to the server;
\item[$\bullet$ \textbf{Step 2}:]
\emph{Model aggregating:} The server performs secure aggregation over the uploaded parameters from $N$ clients without learning local information;
\item[$\bullet$ \textbf{Step 3}:]
\emph{Parameters broadcasting: }The server broadcasts the aggregated parameters to the $N$ clients;
\item[$\bullet$ \textbf{Step 4}:]
\emph{Model updating: }All clients update their respective models with the aggregated parameters and test the performance of the updated models.
\end{basedescript}

In the FL process, the $N$ clients with the same data structure collaboratively learn a ML model with the help of a cloud server.
After a sufficient number of local training and update exchanges between the server and its associated clients, the solution to the optimization problem \eqref{equ:Global objective function} is able to converge to that of the global optimal learning model.
\subsection{Threat Model}
The server in this paper is assumed to be honest. However, there are external adversaries targeting at clients' private information.
Although the individual dataset $\mathcal D_i$ of the $i$-th client is kept locally in FL, the intermediate parameter $\mathbf{w}_{i}$ needs to be shared with the server, which may reveal the clients' private information as demonstrated by model inversion attacks.
For example, authors in~\cite{Fredrikson:2015:MIA:2810103.2813677} demonstrated a model-inversion attack that recovers images from a facial recognition system.
In addition, the privacy leakage can also happen in the broadcasting (through downlink channels) phase by analyzing the global parameter $\mathbf{w}$.

We also assume that uplink channels are more secure than downlink broadcasting channels, since clients can be assigned to different channels (e.g., time slots, frequency bands) dynamically in each uploading time, while downlink channels are broadcasting. Hence, we assume that there are at most $L$ ($L\leq T$) exposures of uploaded parameters from each client in the uplink\footnote{Here we assume that the adversary cannot know where the parameters come from.} and $T$ exposures of aggregated parameters in the downlink, where $T$ is the number of aggregation times.
\subsection{Differential Privacy}\label{sec:Differential Privacy}
$(\epsilon, \delta)$-DP provides a strong criterion for privacy preservation of distributed data processing systems. Here, $\epsilon > 0$ is the distinguishable bound of all outputs on neighboring datasets $\mathcal D_i, \mathcal D_i'$ in a database, and $\delta$ represents the event that the ratio of the probabilities for two adjacent datasets $\mathcal D_i, \mathcal D_i'$ cannot be bounded by $e^{\epsilon}$ after adding a privacy preserving mechanism. With an arbitrarily given $\delta$, a privacy preserving mechanism with a larger $\epsilon$ gives a clearer distinguishability of neighboring datasets and hence a higher risk of privacy violation.
Now, we will formally define DP as follows.
\begin{definition}($(\epsilon, \delta)$-DP~\cite{TCS-042}):
A randomized mechanism $\mathcal M: \mathcal{X}\rightarrow \mathcal{R}$ with domain $\mathcal{X}$ and range $\mathcal{R}$ satisfies $(\epsilon, \delta)$-DP,
if for all measurable sets $\mathcal S\subseteq \mathcal{R}$ and for any two adjacent databases $\mathcal D_i, \mathcal D_i'\in \mathcal{X}$,
\begin{equation}\label{equ:Differential privacy}
\emph{Pr}[\mathcal M(\mathcal D_i)\in \mathcal S]\leq e^{\epsilon}\emph{Pr}[\mathcal M(\mathcal D_i')\in \mathcal S]+\delta.
\end{equation}
\end{definition}

For numerical data, a Gaussian mechanism defined in~\cite{TCS-042} can be used to guarantee $(\epsilon, \delta)$-DP. According to~\cite{TCS-042}, we present the following DP mechanism by adding artificial Gaussian noises.

In order to ensure that the given noise distribution $n\sim \mathcal N(0,\sigma^{2})$ preserves $(\epsilon, \delta)$-DP, where $\mathcal N$ represents the Gaussian distribution, we choose noise scale $\sigma\geq c\Delta s/\epsilon$ and the constant $c\ge\sqrt{2\ln(1.25/\delta)}$ for $\epsilon \in (0,1)$. In this result, $n$ is the value of an additive noise sample for a data in the dateset, $\Delta s$ is the sensitivity of the function $s$ given by $\Delta s =\max_{\mathcal D_i, \mathcal D_i'}{\Vert s(\mathcal D_i)-s(\mathcal D_i')\Vert}$, and $s$ is a real-valued function.

Considering the above DP mechanism, choosing an appropriate level of noise remains a significant research problem, which will affect the privacy guarantee of clients and the convergence rate of the FL process.
\section{Federated Learning with Differential Privacy}\label{sec:FLDP}
In this section, we first introduce the concept of global DP and analyze the DP performance in the context of FL. Then we propose the NbAFL scheme that can satisfy the DP requirement by adding proper noisy perturbations at both the clients and the server.
\subsection{Global Differential Privacy}
Here, we define a global $(\epsilon, \delta)$-DP requirement for both uplink and downlink channels. From the uplink perspective, using a clipping technique, we can ensure that $\Vert\mathbf{w}_{i}\Vert \leq C$, where $\mathbf{w}_{i}$ denotes training parameters from the $i$-th client without perturbation and $C$ is a clipping threshold for bounding $\mathbf{w}_{i}$. We assume that the batch size in the local training is equal to the number of training samples and then define local training process in the $i$-th client by
\begin{multline}\label{localtrain}
s_{\text{U}}^{\mathcal D_i}\triangleq \mathbf{w}_{i}=\arg\min\limits_{\mathbf{w}}{F_{i}(\mathbf{w}, \mathcal D_i)}\\
=\frac{1}{\vert \mathcal D_i \vert}\sum_{j = 1}^{\vert \mathcal D_i \vert}\arg\min\limits_{\mathbf{w}}{F_{i}(\mathbf{w}, \mathcal D_{i,j})},
\end{multline}
where $\mathcal D_i$ is the $i$-th client's database and $\mathcal D_{i,j}$ is the $j$-th sample in $\mathcal D_i$. Thus, the sensitivity of $s_{\text{U}}^{\mathcal D_i}$ can be expressed as
\begin{multline}\label{SensitivityforUP}
\Delta s_{\text{U}}^{\mathcal D_i} =\max_{\mathcal D_i, \mathcal D_i'}{\Vert s_{\text{U}}^{\mathcal D_i}-s_{\text{U}}^{ \mathcal D_i'}\Vert}\\
=\max_{\mathcal D_i, \mathcal D_i'}\left\Vert \frac{1}{\vert \mathcal D_i \vert}\sum_{j = 1}^{\vert \mathcal D_i \vert}\arg\min\limits_{\mathbf{w}}{F_{i}(\mathbf{w}, \mathcal D_{i,j})}\right.\\
\left.-\frac{1}{\vert \mathcal D_i' \vert}\sum_{j = 1}^{\vert \mathcal D_i' \vert}\arg\min\limits_{\mathbf{w}}{F_{i}(\mathbf{w}, \mathcal D_{i,j}')}\right\Vert= \frac{2C}{\vert \mathcal D_i \vert},
\end{multline}
where $\mathcal D_i'$ is an adjacent dataset to $\mathcal D_i$ which has the same size but only differ by one sample, and $\mathcal D_{i,j}'$ is the $j$-th sample in $\mathcal D_i'$.
From the above result, a global sensitivity in the uplink channel can be defined by
\begin{equation}
\Delta s_{\text{U}}\triangleq \max\left\{\Delta s_{\text{U}}^{\mathcal D_i}\right\},\quad\forall i.
\end{equation}
To achieve a small global sensitivity, the ideal condition is that all the clients use sufficient local datasets for training.
Hence, we define the minimum size of the local datasets by $m$ and then obtain $\Delta s_{\text{U}} = \frac{2C}{m}$.
To ensure $(\epsilon, \delta)$-DP for each client in the uplink in one exposure, we set the noise scale, represented by the standard deviation of the additive Gaussian noise, as $\sigma_{\text{U}}= c\Delta s_{\text{U}}/\epsilon$.
Considering $L$ exposures of local parameters, we need to set $\sigma_{\text{U}}= cL\Delta s_{\text{U}}/\epsilon$ due to the linear relation between $\epsilon$ and $\sigma_{\text{U}}$ in the Gaussian mechanism.

From the downlink perspective, the aggregation operation for $\mathcal D_i$ can be expressed as
\begin{multline}\label{DLfunction}
s_{\text{D}}^{\mathcal D_i}\triangleq \mathbf{w}=p_1\mathbf{w}_{1}+\ldots+p_i\mathbf{w}_{i}+\ldots+p_N\mathbf{w}_{N},
\end{multline}
where $1\leq i\leq N$ and $\mathbf{w}$ is the aggregated parameters at the server to be broadcast to the clients. Regarding the sensitivity of $s_{\text{D}}^{\mathcal D_i}$, i.e., $\Delta s_{\text{D}}^{\mathcal D_i}$, we have the following lemma.
\begin{lemma}[Sensitivity after the aggregation operation]\label{lem:SensitivityforSum}
In FL training process, the sensitivity for $\mathcal D_i$ after the aggregation operation $s_{\emph{D}}^{\mathcal D_i}$ is given by
\begin{equation}\label{equ:SensitivityforDL}
\begin{aligned}
\Delta s_{\emph{D}}^{\mathcal D_i} = \frac{2Cp_{i}}{m}.
\end{aligned}
\end{equation}
\end{lemma}
\begin{proof}
See Appendix~\ref{appendix:SensitivityforAgg}.
\end{proof}
\begin{algorithm}
\caption{Noising before Aggregation FL}
\label{alg:NbAFL}
\LinesNumbered
\KwData{$T$, $\mathbf{w}^{(0)}$, $\mu$, $\epsilon$ and $\delta$}
{Initialization: $t = 1$ and $\mathbf{w}^{(0)}_{i} = \mathbf{w}^{(0)}$}, $\forall i$\\
\While {$t \le T$}
{
\textbf{Local training process:}\\
\While {$\mathcal C_i\in \{\mathcal C_1, \mathcal C_2, \ldots,\mathcal C_{N}\}$}
{
Update the local parameters $\mathbf{w}^{(t)}_{i}$ as\\
\quad\quad $\mathbf{w}^{(t)}_{i}=\arg\min\limits_{\mathbf{w}_{i}}{\left(F_{i}(\mathbf{w}_{i})+\frac{\mu}{2}\Vert \mathbf{w}_{i}- \mathbf{w}^{(t-1)}\Vert^{2}\right)}$\\
Clip the local parameters $\mathbf{w}^{(t)}_{i} = \mathbf{w}^{(t)}_{i}/\max\left(1,\frac{\Vert\mathbf{w}^{(t)}_{i}\Vert}{C}\right)$\\
Add noise and upload parameters $\widetilde{\mathbf{w}}^{(t)}_{i}=\mathbf{w}^{(t)}_{i}+\mathbf{n}^{(t)}_{i}$\\
}
\textbf{Model aggregating process:}\\
Update the global parameters $\mathbf{w}^{(t)}$ as\\
\quad\quad $\mathbf{w}^{(t)} = \sum\limits_{i=1}^{N}{p_{i}\widetilde{\mathbf{w}}^{(t)}_{i}}$\\
The server broadcasts global noised parameters \\
\quad\quad$\widetilde{\mathbf{w}}^{(t)}=\mathbf{w}^{(t)}+\mathbf{n}_{\text D}^{(t)}$\\
\textbf{Local testing process:}\\
\While {$\mathcal C_i\in \{\mathcal C_1, \mathcal C_2, \ldots,\mathcal C_{N}\}$}
{
Test the aggregating parameters $\widetilde{\mathbf{w}}^{(t)}$ using local dataset\\
}
$t \leftarrow t + 1$
}
\KwResult{$\widetilde{\mathbf{w}}^{(T)}$}
\end{algorithm}

\begin{remark}
From the above lemma, to achieve a small global sensitivity in the downlin channel which is defined by
\begin{equation}
\Delta s_{\emph{D}}\triangleq \max\left\{ \Delta s_{\emph{D}}^{\mathcal D_i}\right\}=\max\left\{ \frac{2Cp_{i}}{m}\right\},\quad \forall i,
\end{equation}
the ideal condition is that all the clients should use the same size of local datasets for training, i.e., $p_{i}=1/N$.
\end{remark}

From the above remark, when setting $p_{i}=1/N$, $\forall i$, we can obtain the optimal value of the sensitivity $\Delta s_{\text{D}}$.
So here we should add noise at the client side first and then decide whether or not to add noises at server to satisfy the $(\epsilon, \delta)$-DP criterion in the downlink channel.

\begin{theorem}[DP guarantee for downlink channels]\label{theorem:DPforDownCh}
To ensure $(\epsilon, \delta)$-DP in the downlink channels with $T$ aggregations, the standard deviation of Gaussian noises $\mathbf{n}_{\emph D}$ that are added to the aggregated parameter $\textbf{w}$ by the server can be given as
\begin{equation}\label{equ:NoiseScaleinServer}
\sigma_{\emph{D}} =
\begin{cases}
\frac{2cC\sqrt{T^{2}-L^{2}N}}{mN\epsilon} & T>L\sqrt{N},\\
0&T\leq L\sqrt{N}.
\end{cases}
\end{equation}
\end{theorem}
\begin{proof}
See Appendix~\ref{appendix:DPforDownCh}.
\end{proof}

\textbf{Theorem~\ref{theorem:DPforDownCh}} shows that to satisfy a $(\epsilon, \delta)$-DP requirement for the downlink channels, additional noises $\mathbf{n}_{\text D}$ need to be added by the server.
With a certain $L$, the standard deviation of additional noises is depending on the relationship between the number of aggregation times $T$ and the number of clients $N$. The intuition is that a larger $T$ can lead to a higher chance of information leakage, while a larger number of clients is helpful for hiding their private information. This theorem also provides the variance value of the noises that should be added to the aggregated parameters. Based on the above results, we propose the following NbAFL algorithm.
\subsection{Proposed NbAFL}
\textbf{Algorithm~\ref{alg:NbAFL}} outlines our NbAFL for training an effective model with a global $(\epsilon, \delta)$-DP requirement. We denote by $\mu$ the presetting constant of the proximal term and by $\mathbf{w}^{(0)}$ the initiate global parameter. At the beginning of this algorithm, the server broadcasts the required privacy level parameters $(\epsilon, \delta)$ are set and the initiate global parameter $\mathbf{w}^{(0)}$ are sent to clients. In the $t$-th aggregation, $N$ active clients respectively train the parameters by using local databases with preset termination conditions. After completing the local training, the $i$-th client, $\forall i$, will add noises to the trained parameters $\mathbf{w}^{(t)}_i$, and upload the noised parameters $\widetilde{\mathbf{w}}^{(t)}_i$ to the server for aggregation.

Then the server update the global parameters ${\mathbf{w}}^{(t)}$ by aggregating the local parameters integrated with different weights. Additive noises $\textbf{n}^{(t)}_{\text{D}}$ are added to this ${\mathbf{w}}^{(t)}$ according to \textbf{Theorem~\ref{theorem:DPforDownCh}} before being broadcast to the clients. Based on the received global parameters $\widetilde{\mathbf{w}}^{(t)}$, each client will estimate the accuracy by using local testing databases and start the next round of training process based on these received parameters. The FL process completes after the aggregation time reaches a preset number $T$ and the algorithm returns $\widetilde{\mathbf{w}}^{(T)}$.

Now, let us focus on the privacy preservation performance of the NbAFL. First, the set of all local parameters, denoted by $\widetilde{\mathbf{W}} = \{\widetilde{\mathbf{w}}_{1}, \ldots,\widetilde{\mathbf{w}}_{N}\}$, are received by the server. Owing to the local perturbations in the NbAFL, it will be difficult for malicious adversaries to infer the information at the $i$-client from its uploaded parameters $\widetilde{\mathbf{w}}_{i}$. After the model aggregation, the aggregated parameters $\mathbf{w}$ will be sent back to clients via broadcast channels. This poses threats on clients's privacy as potential adversaries may reveal sensitive information about individual clients from $\mathbf{w}$. In this case, additive noises may be posed to $\mathbf{w}$ based on \textbf{Theorem~\ref{theorem:DPforDownCh}}.
\section{Convergence Analysis on NbAFL}\label{sec:Con_NbAFL}
In this section, we are ready to analyze the convergence performance of the proposed NbAFL. First, we analyze the expected increment of adjacent aggregations in the loss function with Gaussian noises. Then, we focus on deriving the convergence property under the global $(\epsilon, \delta)$-DP requirement.

For the convenience of the analysis, we make the following assumptions on the loss function and network parameters.
\begin{assumption}\label{ass:LossFunction}
We make assumptions on the global loss function $F(\cdot)$ defined by $F(\cdot)\triangleq \sum_{i}^{N}p_iF_{i}(\cdot)$, and the $i$-th local loss function $F_{i}(\cdot)$ as follows:
\begin{enumerate}
\item[\emph{1)}] $F_{i}(\mathbf{w})$ is convex;
\item[\emph{2)}] $F_{i}(\mathbf{w})$ satisfies the Polyak-Lojasiewicz condition with the positive parameter $l$, which implies that
$F(\mathbf{w})-F(\mathbf{w}^{*})\leq\frac{1}{2l}\Vert \nabla F(\mathbf{w})\Vert^{2}$, where $\mathbf{w}^{*}$ is the optimal result;
\item[\emph{3)}] $F(\mathbf{w}^{(0)})-F(\mathbf{w}^{*}) = \Theta$;
\item[\emph{4)}] $F_{i}(\mathbf{w})$ is $\beta$-Lipschitz, i.e., $\Vert F_{i}(\mathbf{w})- F_{i}(\mathbf{w}')\Vert\leq \beta\Vert \mathbf{w}-\mathbf{w}'\Vert$, for any $\mathbf{w}$, $\mathbf{w}'$;
\item[\emph{5)}] $F_{i}(\mathbf{w})$ is $\rho$-Lipschitz smooth, i.e., $\Vert \nabla F_{i}(\mathbf{w})-\nabla F_{i}(\mathbf{w}')\Vert\leq \rho\Vert \mathbf{w}-\mathbf{w}'\Vert$, for any $\mathbf{w}$, $\mathbf{w}'$, where $\rho$ is a constant determined by the practical loss function;
\item[\emph{6)}] For any $i$ and $\mathbf{w}$, $\Vert\nabla F_{i}(\mathbf{w})-\nabla F(\mathbf{w})\Vert\leq \varepsilon_{i}$, where $\varepsilon_{i}$ is the divergence metric.
\end{enumerate}
\end{assumption}

Similar to the gradient divergence, the divergence metric $\varepsilon_{i}$ is the metric to capture the divergence between the gradients of the local loss functions and that of the aggregated loss function, which is essential for analyzing SGD. The divergence is related to how the data is distributed at different nodes. Using \textbf{Assumption~1}, we then have the following lemma.
\begin{lemma}[$B$-dissimilarity of various clients]\label{lem:Bdissimilarity}
For a given ML parameter $\mathbf{w}$, there exists $B$ satisfying
\begin{equation}
\mathbb{E}\left\{ \Vert\nabla F_{i}(\mathbf{w})\Vert^{2}\right\}\leq\Vert\nabla F(\mathbf{w})\Vert^{2}B^{2},\quad{\forall i}.
\end{equation}
\end{lemma}
\begin{proof}
See Appendix~\ref{appendix:B_dissimilarity}.
\end{proof}

\textbf{Lemma~\ref{lem:Bdissimilarity}} comes from the assumption of the divergence metric and demonstrates the statistical heterogeneity of all clients. As mentioned earlier, the values of $\rho$ and $B(\mathbf{w})$ are determined by the specific global loss function $F(\mathbf{w})$ in practice and the training parameters $\mathbf{w}$. With the above preparation, we are now ready to analyze the convergence property of NbAFL. First, we present the following lemma to derive an expected increment bound on the loss function during each iteration of parameters with artificial noises.
\begin{lemma}[Expected increment in the loss function]\label{theorem:Expincrement}
After receiving updates, from the $t$-th to the $(t+1)$-th aggregation, the expected difference in the loss function can be upper-bounded by
\begin{multline}\label{equ:theorem2-1}
\mathbb{E}\{F(\widetilde{\mathbf{w}}^{(t+1)})-F(\widetilde{\mathbf{w}}^{(t)})\}\leq\lambda_{2}\mathbb{E}\{\Vert\nabla F(\widetilde{\mathbf{w}}^{(t)})\Vert^{2}\}\\
 + \lambda_{1}\mathbb{E}\{\Vert\mathbf{n}^{(t+1)}\Vert\Vert\nabla F(\widetilde{\mathbf{w}}^{(t)})\Vert\}+ \lambda_{0}\mathbb{E}\{\Vert\mathbf{n}^{(t+1)}\Vert^{2}\},
\end{multline}
where
\begin{equation}
\lambda_{0} = \frac{\rho}{2},\,\, \lambda_{1} = \frac{1}{\mu}+\frac{\rho B}{\mu},
\end{equation}
\begin{equation}
\lambda_{2}=-\frac{1}{\mu}+\frac{\rho B}{\mu^{2}}+\frac{\rho B^{2}}{2\mu^{2}},
\end{equation}
and $\mathbf{n}^{(t)}$ are the equivalent noises imposed on the parameters after the $t$-th aggregation, given by
\begin{equation}
\mathbf{n}^{(t)} = \sum_{i=1}^{N}{p_{i}\mathbf{n}_{i}^{(t)}}+\mathbf{n}^{(t)}_{\emph{D}}.
\end{equation}
\end{lemma}
\begin{proof}
See Appendix~\ref{appendix:ConvforN}.
\end{proof}

In this lemma, the value of an additive noise sample $n$ in vector $\mathbf{n}^{(t)}$ satisfies the following Gaussian distribution $n \sim \mathcal N(0,\sigma_\text{A}^{2})$. Also, we can obtain $\sigma_{\text{A}} = \sqrt{\sigma_{\text{D}}^{2}+\sigma_{\text{U}}^{2}/N}$ from Section~\ref{sec:FLDP}. From the right hand side (RHS) of the above inequality, we can see that it is crucial to select a proper proximal term $\mu$ to achieve a low upper-bound. It is clear that artificial noises with a large $\sigma_\text{A}$ may improve the DP performance in terms privacy protection. However, from the RHS of (\ref{equ:theorem2-1}), a large $\sigma_\text{A}$ may enlarge the expected difference of the loss function between two consecutive aggregations, leading to a deterioration of convergence performance.

Furthermore, to satisfy the global $(\epsilon, \delta)$-DP, by using~\textbf{Theorem~\ref{theorem:DPforDownCh}}, we have
\begin{equation}\label{equ:theorem2-2}
\sigma_{\text{A}} =
\begin{cases}
\frac{cT\Delta s_{\text{D}}}{\epsilon} & T>L\sqrt{N},\\
\frac{cL\Delta s_{\text{U}}}{\sqrt{N}\epsilon}&T\leq L\sqrt{N}.
\end{cases}
\end{equation}
Next, we will analyze the convergence property of NbAFL with the $(\epsilon,\delta)$-DP requirement.

\begin{theorem}[Convergence upper bound of the NbAFL]\label{theorem:ConvUpperBound}
With required protection level $\epsilon$, the convergence upper bound of \textbf{\emph{Algorithm~\ref{alg:NbAFL}}} after $T$ aggregations is given by
\begin{multline}\label{equ:Upper_bound_DP}
\mathbb{E}\{F(\widetilde{\mathbf{w}}^{(T)})-F(\mathbf{w}^{*})\}
\leq \\P^{T}\Theta+\left(\frac{\kappa_{1}T}{\epsilon}+\frac{\kappa_{0}T^{2}}{\epsilon^{2}}\right)\left(1-P^T\right),
\end{multline}
where
\begin{equation}
P = 1+2l\lambda_{2}, \,\,\kappa_{1} = \frac{\lambda_{1}\beta c}{m(1-P)}\sqrt{\frac{2}{N\pi}}
\end{equation}
and
\begin{equation}
\kappa_{0} = \frac{\lambda_{0}c^{2}}{m^{2}(1-P)N}.
\end{equation}
\end{theorem}
\begin{proof}
See Appendix D.
\end{proof}

\textbf{Theorem~\ref{theorem:ConvUpperBound}} reveals an important relationship between privacy and utility by taking into account the protection level $\epsilon$ and the number of aggregation times $T$.
As the number of aggregation times $T$ increases, the first term of the upper bound decreases but the second term increases. Furthermore, By viewing $T$ as a continuous variable and by writing the RHS of (\ref{equ:Upper_bound_DP}) as $h(T)$, we have
\begin{multline}\label{equ:second_order_derivative}
\frac{d^{2}h(T)}{d^{2}T} =\left(\Theta-\frac{\kappa_{1}T}{\epsilon}-\frac{\kappa_{0}T^{2}}{\epsilon^{2}}\right)P^{T}\ln^{2}{P}\\
-2\left(\frac{\kappa_{1}}{\epsilon}+\frac{2\kappa_{0}T}{\epsilon^{2}}\right)P^T\ln{P}+\frac{2\kappa_{0}}{\epsilon^{2}}\left(1-P^T\right).
\end{multline}
It can be seen that the second term and third term of on the RHS of (\ref{equ:second_order_derivative}) are always positive. When $N$ and $\epsilon $ are set to be large enough, we can see that $\kappa_{1}$ and $\kappa_{0}$ are small, and thus the first term can also be positive. In this case, we have $d^{2}h(T)/d^{2}T>0$ and the upper bound is convex for $T$.

\begin{remark}\label{remark:ConvforN_epsilon}
As can be seen from this theorem, expected gap between the achieved loss function $F(\widetilde {\mathbf{ w}}^{(T)})$ and the minimum one $F(\mathbf{w}^{*})$ is a decreasing function of $\epsilon$. By increasing $\epsilon$, i.e., relaxing the privacy protection level, the performance of NbAFL algorithm will improve. This is reasonable because the variance of artificial noises decreases, thereby improving the convergence performance.
\end{remark}

\begin{remark}\label{remark:ConvforN_N}
The number of clients $N$ will also affect its iterative convergence performance, i.e., a larger $N$ would achieve a better convergence performance. This is because a lager $N$ leads to a lower variance of the artificial noises.
\end{remark}

\begin{remark}\label{remark:ConvforN_T}
There is an optimal number of maximum aggregation times $T$ in terms of convergence performance for given $\epsilon$ and $N$. In more detail, a larger $T$ may lead to a higher variance of artificial noises, and thus pose a negative impact on convergence performance. On the other hand, more iterations can generally boost the convergence performance if noises are not large enough. In this sense, there is a tradeoff on choosing a proper $T$.
\end{remark}
\section{$K$-Client Random Scheduling Policy}\label{Sec:K-Random}
In this section, we consider the case where only $K (K<N)$ clients are selected to participate in the aggregation process, namelly $K$-random scheduling.

We now discuss how to add artificial noises in the $K$-random scheduling to satisfy a global $(\epsilon, \delta)$-DP. It is nature that in the uplink channels, each of the $K$ scheduled clients should add noises with scale $\sigma_{\text{U}}=cL\Delta s_{\text{U}}/\epsilon$ for achieving $(\epsilon, \delta)$-DP. This is equivalent to the noise scale in the all-clients selection case in~\textbf{Section~\ref{sec:FLDP}}, since each client only considers its own privacy for uplink channels in both cases. However, the derivation of the noise scale in the downlink will be different for the $K$-random scheduling. As an extension of \textbf{Theorem~1}, we present the following lemma in the case of $K$-random scheduling on how to obtain $\sigma_{\text{D}}$.
\begin{lemma}[DP guarantee in $K$-random scheduling]\label{lem:NoiseKSche}
In the NbAFL algorithm with $K$-random scheduling, to satisfy a global $(\epsilon, \delta)$-DP, and the standard deviation $\sigma_{\emph{D}}$ of additive Gaussian noises for downlink channels should be set as
\begin{equation}
\begin{aligned}
\sigma_{\emph{D}}=
\begin{cases}
\frac{2cC\sqrt{\frac{T^{2}}{b^{2}}-L^{2}K}}{mK\epsilon} & T>\frac{\epsilon}{\gamma},\\
0&T\leq \frac{\epsilon}{\gamma},
\end{cases}
\end{aligned}
\end{equation}
where
\begin{equation}
\begin{split}
&b=-\frac{T}{\epsilon}\ln\left(1-\frac{N}{K}+\frac{N}{K}e^{\frac{-\epsilon}{T}}\right),\\
&\gamma = -\ln\left({1-\frac{K}{N}+ \frac{K}{N}e^{\frac{-\epsilon}{L\sqrt{K}}}}\right).
\end{split}
\end{equation}
\end{lemma}
\begin{IEEEproof}
See Appendix~\ref{appendix:DPforKshe}.
\end{IEEEproof}

\textbf{Lemma~\ref{lem:NoiseKSche}} recalculates $\sigma_{\text{D}}$ by considering the number of chosen clients $K$.
Generally, the number of clients $N$ is fixed, we thus focus on the effect of $K$.
Based on the DP analysis in~\textbf{Lemma~\ref{lem:NoiseKSche}}, we can obtain the following theorem.
\begin{theorem}[Convergence under $K$-random scheduling]\label{theorem:ConvKSche}
With required protection level $\epsilon$ and the number of chosen clients $K$, for any $\Theta>0$, the convergence upper bound after $T$ aggregation times is given by
\begin{equation}\label{equ:theorem-3}
\begin{aligned}
&\mathbb{E}\{F(\widetilde{\mathbf{v}}^{T})-F(\mathbf{w}^{*})\}
\leq Q^{T}\Theta \\
&\quad+ \frac{1-Q^{T}}{1-Q}\left(\frac{c\alpha_{1}\beta}{-mK\ln\left(1-\frac{N}{K}+\frac{N}{K}e^{-\frac{\epsilon}{T}}\right)}\sqrt{\frac{2}{\pi}}\right.\\
&\quad+\left.\frac{c^{2}\alpha_{0}}{m^{2}K^{2}\ln^{2}\left(1-\frac{N}{K}+\frac{N}{K}e^{-\frac{\epsilon}{T}}\right)}\right).
\end{aligned}
\end{equation}
where
\begin{equation}
Q=1+\frac{2l}{\mu^{2}}\left(\frac{\rho B^{2}}{2}+\rho B+\frac{\rho B^{2}}{K}+\frac{2\rho B^{2}}{\sqrt{K}}+\frac{\mu B}{\sqrt{K}}-\mu\right),
\end{equation}
\begin{multline}
\alpha_{0} = \frac{2\rho K}{N}+\rho,\,\,\alpha_{1} = 1+\frac{2\rho B}{\mu}+\frac{2\rho B\sqrt{K}}{\mu N},
\end{multline}
and
\begin{equation}
\widetilde{\mathbf{v}}^{(T)}=\sum_{i=1}^{K}{p_{i}\left(\mathbf{w}^{(T)}_{i}+\mathbf{n}^{(T)}_{i}\right)}+\mathbf{n}^{(T)}_{\emph{D}}.
\end{equation}
\end{theorem}
\begin{proof}
See Appendix~\ref{appendix:ConvforK}.
\end{proof}

The above theorem provides the convergence upper bound between $F(\widetilde{\mathbf{v}}^{T})$ and $F(\mathbf{w}^{*})$ under $K$-random scheduling.
Using $K$-random scheduling, we can obtain an important relationship between privacy and utility by taking into account the protection level $\epsilon$, the number of aggregation times $T$ and the number of chosen clients $K$.

\begin{remark}\label{remark:ConvforK_K}
From the bound derived in~\emph{\textbf{Theorem~\ref{theorem:ConvKSche}}}, we conclude that there is an optimal $K$ in between $0$ and $N$ that achieves the optimal convergence performance. That is, by finding a proper $K$, the $K$-random scheduling policy is superior to the one that all $N$ clients participate in the FL aggregations.
\end{remark}

\section{Simulation Results}\label{Sec:Exm_Res}
In this section, we evaluate the proposed NbAFL by using multi-layer perception (MLP) and real-world federated datasets.
In order to characterize the convergence property of NbAFL, we conduct experiments by varying the protection levels of $\epsilon$, the number of clients $N$, the number of maximum aggregation times $T$ and the number of chosen clients $K$.

We conduct experiments on the standard MNIST dataset for handwritten digit recognition consisting of $60000$ training examples and $10000$ testing examples~\cite{726791}.
Each example is a $28 \times 28$ size gray-level image.
Our baseline model uses a a MLP network with a single hidden layer containing 256 hidden units.
In this feed-forward neural network, we use a ReLU units and softmax of $10$ classes (corresponding to the $10$ digits) with the cross-entropy loss function.
For the optimizer of networks, we set the learning rate to $0.002$.
The values of $\rho$, $\beta$, $l$ and $B$ are determined by the specific loss function, and we will use estimated values in our simulations~\cite{8664630}.

\subsection{Performance Evaluation on Protection Levels}
\begin{figure}[htb]
\centering
\includegraphics[width=3.5in,angle=0]{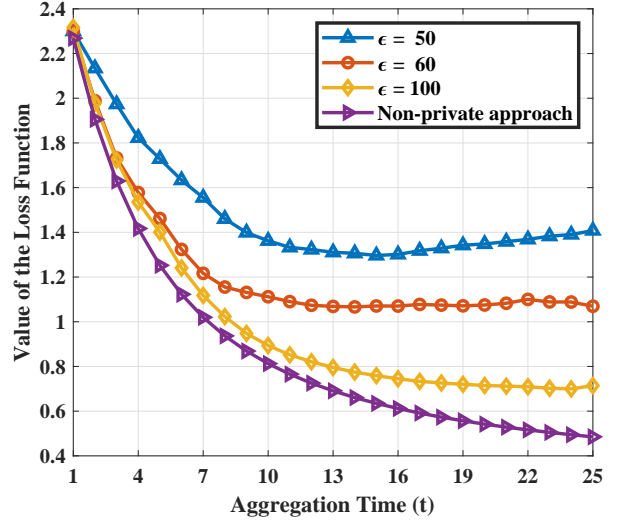}
\caption{The comparison of training loss with various protection levels for 50 clients using $\epsilon = 50$, $\epsilon = 60$ and $\epsilon = 100$, respectively.}
\label{fig:ConvforN_epsilon_loss}
\end{figure}
\begin{figure}[htb]
\centering
\includegraphics[width=3.5in,angle=0]{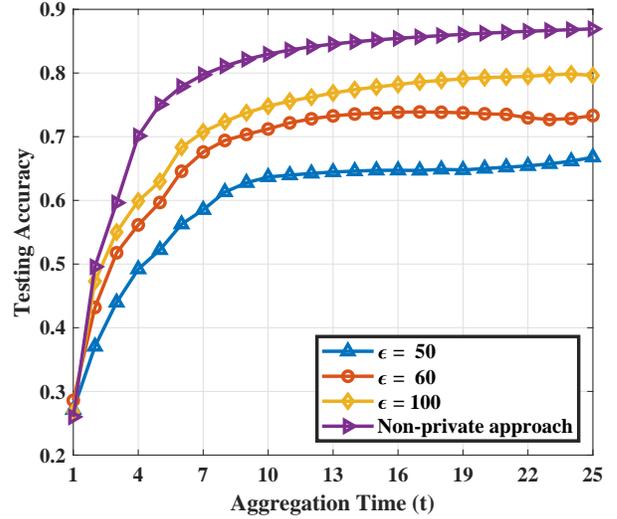}
\caption{The comparison of training accuracy with various protection levels for 50 clients using $\epsilon = 50$, $\epsilon = 60$ and $\epsilon = 100$, respectively.}
\label{fig:ConvforN_epsilon_acc}
\end{figure}
In Fig.~\ref{fig:ConvforN_epsilon_loss} and Fig.~\ref{fig:ConvforN_epsilon_acc}, we choose various protection levels $\epsilon=50$, $\epsilon=60$ and $\epsilon=100$ to show the results of the loss function and testing accuracies in NbAFL.
Furthermore, we also include a non-private approach to compare with our NbAFL.
In this experiment, we set $N = 50$, $T = 25$ and $\delta = 0.01$, and compute the values of the loss function as a function of the aggregation times $T$.
As shown in Fig.~\ref{fig:ConvforN_epsilon_loss}, values of the loss function in NbAFL are decreasing as we relax the privacy guarantees (increasing $\epsilon$).
Meanwhile, in Fig.~\ref{fig:ConvforN_epsilon_acc}, testing accuracies are also increasing as the privacy parameter reduces. Such observation results are in line with~\textbf{Remark~\ref{remark:ConvforN_epsilon}}.

\begin{figure}[htb]
\centering
\includegraphics[width=3.5in,angle=0]{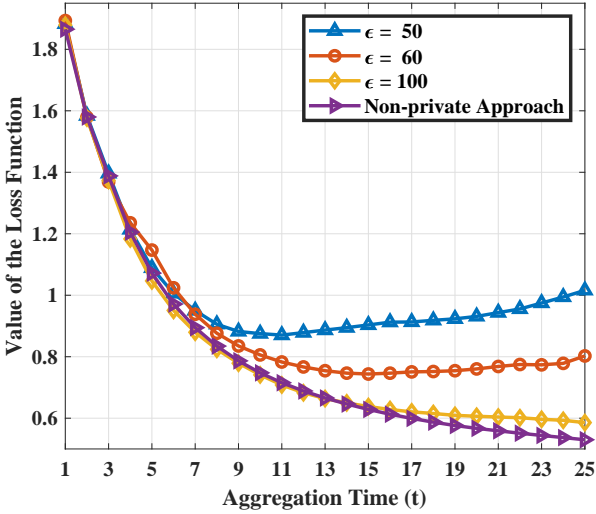}
\caption{The comparison of training loss with various privacy levels for 50 clients using $\epsilon = 50$, $\epsilon = 60$ and $\epsilon = 100$, respectively.}
\label{fig:ConvforK_epsilon_loss}
\end{figure}
\begin{figure}[htb]
\centering
\includegraphics[width=3.5in,angle=0]{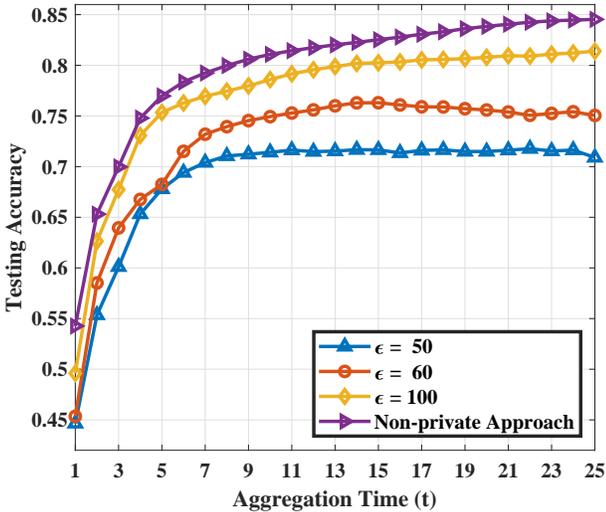}
\caption{The comparison of training accuracy with various privacy levels for 50 clients using $\epsilon = 50$, $\epsilon = 60$ and $\epsilon = 100$, respectively.}
\label{fig:ConvforK_epsilon_acc}
\end{figure}

Considering the $K$-client random scheduling, in Fig.~\ref{fig:ConvforK_epsilon_loss} and Fig.~\ref{fig:ConvforK_epsilon_acc}, we investigate the performances with various protection levels $\epsilon=50$, $\epsilon=60$ and $\epsilon=100$.
For simulation parameters, we set $N = 50$, $K = 20$, $T = 25$, and $\delta = 0.01$.
As shown in Fig.~\ref{fig:ConvforK_epsilon_loss} and Fig.~\ref{fig:ConvforK_epsilon_acc}, the convergence performance under the $K$-client random scheduling is improved with an increasing $\epsilon$.
\subsection{Impact of the number of clients $N$}
\begin{figure}[htb]
\centering
\includegraphics[width=3.5in,angle=0]{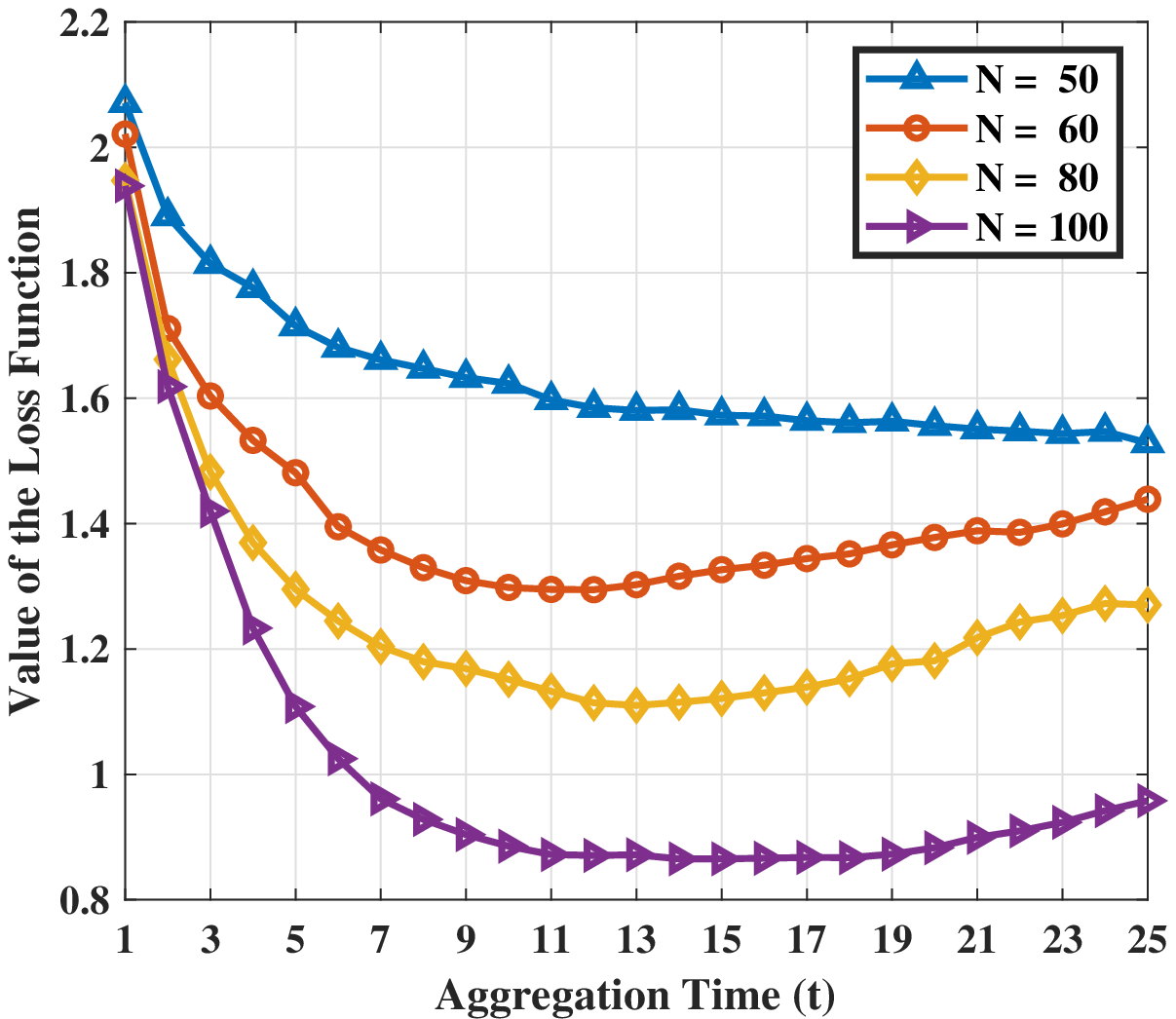}
\caption{The value of the loss function with various numbers of clients under $\epsilon = 60$ under NbAFL Algorithm with $50$ clients.}
\label{fig:ConvforN_N_loss}
\end{figure}
\begin{figure}[htb]
\centering
\includegraphics[width=3.5in,angle=0]{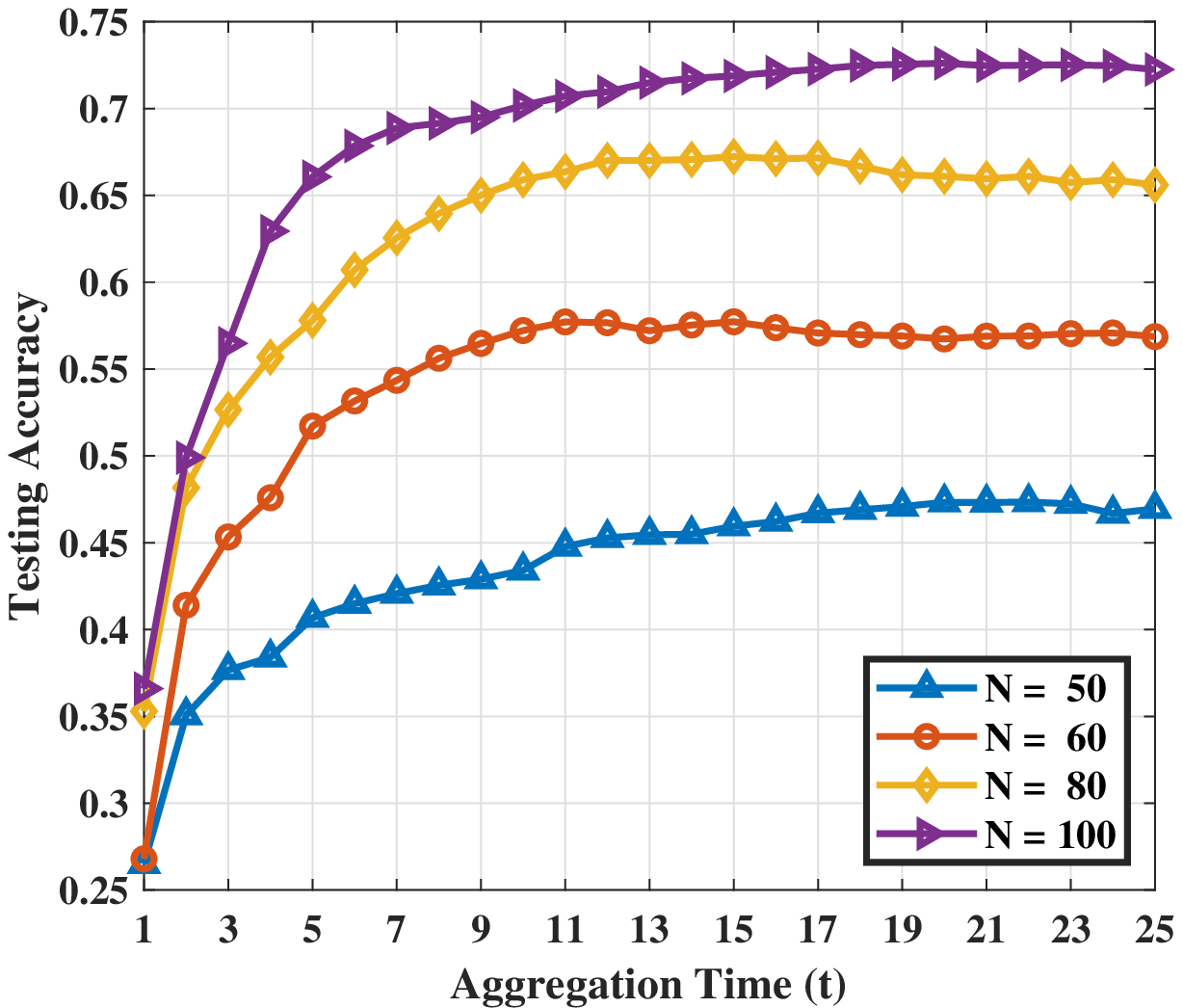}
\caption{The value of the loss function with various numbers of clients under $\epsilon = 60$ under NbAFL Algorithm with $50$ clients.}
\label{fig:ConvforN_N_acc}
\end{figure}

Fig.~\ref{fig:ConvforN_N_loss} and Fig.~\ref{fig:ConvforN_N_acc} compare the convergence performance of NbAFL under required protection level $\epsilon = 60$ and $\delta = 10^{-2}$ as a function of clients' number, $N$.
In this experiment, we set $N = 50$, $N = 60$, $N = 80$ and $N = 100$.
We notice that the performance among different numbers of clients is governed by~\textbf{Remark~\ref{remark:ConvforN_N}}.
This is because that more clients not only provide larger global datasets for training, but also bring down the of standard deviation additive noises due to the aggregation.
\subsection{Impact of the number of maximum aggregation times $T$}
\begin{figure}[htb]
\centering
\includegraphics[width=3.5in,angle=0]{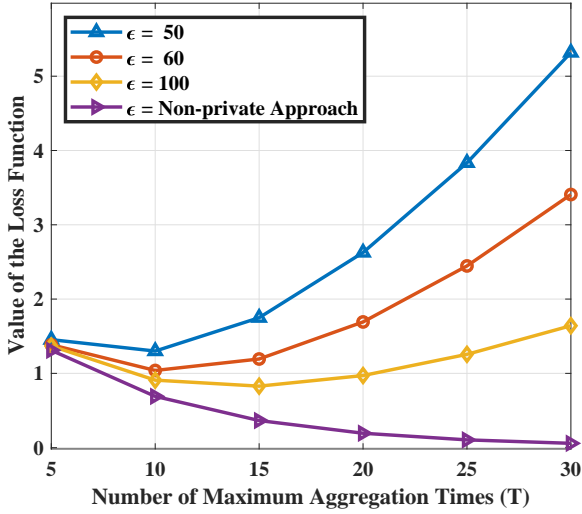}
\caption{The convergence upper bounds with various privacy levels $\epsilon = 50$, $60$ and $100$ under 50-clients' NbAFL algorithm.}
\label{fig:ConvforN_NumericalResult}
\end{figure}
In Fig. \ref{fig:ConvforN_NumericalResult}, we show the theoretical upper bound of training loss as a function of maximum aggregation times with various privacy levels $\epsilon = 50$, $60$ and $100$ under NbAFL algorithm.
Fig. \ref{fig:ConvforN_NumRe_ExpRe} compares the theoretical upper bound using the dotted line and experimental results using the solid line with $\epsilon = 60$ and $100$.
Fig.~\ref{fig:ConvforN_NumericalResult} and Fig. \ref{fig:ConvforN_NumRe_ExpRe} reveal that under a low privacy level (a large $\epsilon$), NbAFL gives a large improvement in terms of the convergence performance.
This observation is in line with \textbf{Remark~\ref{remark:ConvforN_T}}, and the reason comes from the fact that a lower privacy level decreases the standard deviation of additive noises and the server can obtain better quality ML model parameters from the clients.
Fig.~\ref{fig:ConvforN_NumericalResult} also implies that an optimal number of maximum aggregation times increases almost with respect to the increasing $\epsilon$.

\begin{figure}[htb]
\centering
\includegraphics[width=3.5in,angle=0]{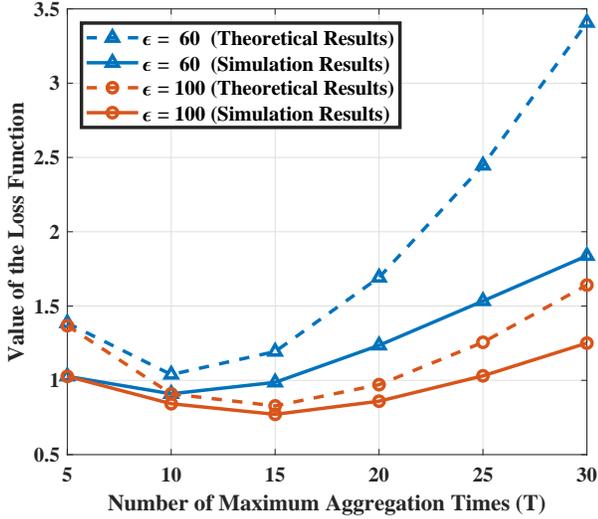}
\caption{The comparison of the loss function between experimental and theoretical results with the various aggregation times under NbAFL Algorithm with $50$ clients.}
\label{fig:ConvforN_NumRe_ExpRe}
\end{figure}
\begin{figure}[htb]
\centering
\includegraphics[width=3.5in,angle=0]{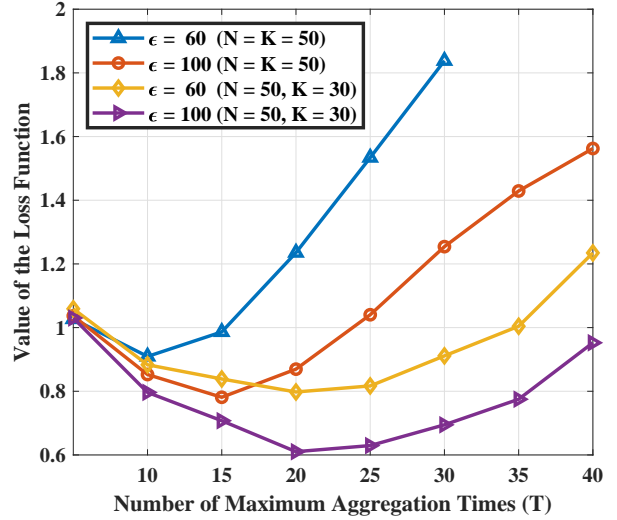}
\caption{The value of the loss function with various privacy levels $\epsilon = 60$ and $\epsilon = 80$ under NbAFL Algorithm with $50$ clients.}
\label{fig:ConvforK_T_loss}
\end{figure}

Fig.~\ref{fig:ConvforK_T_loss} compares the normal NbAFL and $K$-random scheduling based NbAFL for a given protection level.
In Fig.~\ref{fig:ConvforK_T_loss}, we plot the values of the loss function in NbAFL with various numbers of maximum aggregation times.
This figure shows that the value of loss function is a convex function of maximum aggregation times for a given protection leavel under NbAFL algorithm, which validates~\textbf{Remark~\ref{remark:ConvforN_T}}.
From Fig.~\ref{fig:ConvforK_T_loss}, we can also see that for a given $\epsilon$, $K$-random scheduling based NbAFL algorithm has a better convergence performance than the normalized NbAFL algorithm for a larger $T$.
This is because that $K$-random scheduling can bring down the variance of artificial noises with little performance loss.
\subsection{Impact of the number of chosen clients $K$}
\begin{figure}[htb]
\centering
\includegraphics[width=3.5in,angle=0]{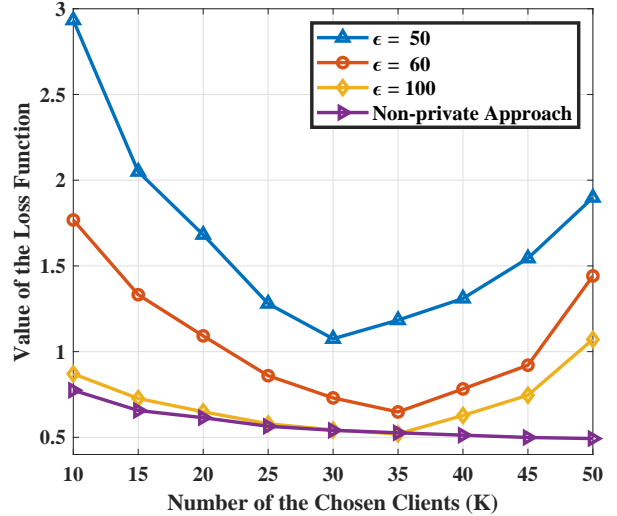}
\caption{The value of the loss function with various numbers of chosen clients under $\epsilon = 50, 60, 100$ under NbAFL Algorithm and non-private approach with $50$ clients.}
\label{fig:ConvforK_K_loss}
\end{figure}

In Fig.~\ref{fig:ConvforK_K_loss}, we plot values of the loss function with various numbers of chosen clients $K$ under the random scheduling policy in NbAFL.
The number of clients is $N = 50$, and $K$ clients are randomly chosen to participate in training and aggregation in each iteration.
In this experiment, we set $\epsilon= 50$, $\epsilon = 60$, $\epsilon = 80$ and $\delta = 0.01$.
Meanwhile, we also exhibit the performance of the non-private approach with various numbers of chosen clients $K$.
Note that an optimal $K$ which further improves the convergence performance exists for various protection levels, due to a trade-off between enhance privacy protection and involving larger global training datasets in each model updating round.
This observation is in line with~\textbf{Remark~\ref{remark:ConvforK_K}}.
The figure shows that in NbAFL, for a given protection level $\epsilon$, the $K$-random scheduling can obtain a better tradeoff than the normal selection policy.
\section{Conclusions}\label{Sec:Concl}
In this paper, we have focused on differential attacks in SGD based FL.
We first define a global $(\epsilon, \delta)$-DP requirement for both uplink and downlink channels, and develop variances of artificial noises at clients and server sides.
Then, we propose a novel framework based on the concept of global $(\epsilon, \delta)$-DP, named NbAFL.
We develop theoretically a convergence bound of the loss function of the trained FL model in the NbAFL.
From theoretical convergence bounds, we obatin the following results: 1) There is a tradeoff between the convergence performance and privacy protection levels, i.e., a better convergence performance leads to a lower protection level; 2) Increasing the number $N$ of overall clients participating in FL can improve the convergence performance, given a fixed privacy protection level; 3) There is an optimal number of maximum aggregation times in terms of convergence performance for a given protection level.
Furthermore, we propose a $K$-random scheduling strategy and also develop the corresponding convergence bound of the loss function in this case.
In addition to above three properties. we find that there exists an optimal value of $K$ that achieves the best convergence performance at a fixed privacy level.
Extensive simulation results confirm the correctness of our analysis.
Therefore, our analytical results are helpful for the design on privacy-preserving FL architectures with different tradeoff requirements on convergence performance and privacy levels.
\appendices
\section{Proof of Lemma~\ref{lem:SensitivityforSum}}\label{appendix:SensitivityforAgg}
From the downlink perspective, for all $\mathcal D_i$ and $\mathcal D_i'$ which differ in a signal entry, the sensitivity can be expressed as
\begin{equation}
\Delta s_{\text{D}}^{\mathcal D_i}=\max_{\mathcal D_i, \mathcal D_i'}\Vert s_{\text{D}}^{\mathcal D_i}-s_{\text{D}}^{\mathcal D_i'}\Vert.
\end{equation}
Based on~\eqref{localtrain} and~\eqref{DLfunction}, we have
\begin{equation}
s_{\text{D}}^{\mathcal D_i}=p_1\mathbf{w}_{1}(\mathcal D_1)+\ldots+p_i\mathbf{w}_{i}(\mathcal D_i)+\ldots+p_N\mathbf{w}_{N}(\mathcal D_N)
\end{equation}
and
\begin{equation}
s_{\text{D}}^{\mathcal D_i'}=p_1\mathbf{w}_{1}(\mathcal D_1)+\ldots+p_i\mathbf{w}_{i}(\mathcal D_i')+\ldots+p_N\mathbf{w}_{N}(\mathcal D_N),
\end{equation}
Furthermore, the sensitivity can be given as
\begin{multline}
\Delta s_{\text{D}}^{\mathcal D_i}=\max_{\mathcal D_i, \mathcal D_i'}\Vert p_i\mathbf{w}_{i}(\mathcal D_i)- p_i\mathbf{w}_{i}(\mathcal D_i')\Vert\\
p_{i}\max_{\mathcal D_i, \mathcal D_i'}\Vert \mathbf{w}_{i}(\mathcal D_i)-\mathbf{w}_{i}(\mathcal D_i')\Vert = p_{i}\Delta s_{\text{U}}^{\mathcal D_i}\leq\frac{2Cp_{i}}{m}.
\end{multline}
Hence, we can set $\Delta s_{\text{D}}^{\mathcal D_i}= \frac{2Cp_{i}}{m}$.
This completes the proof. $\hfill\square$
\section{Proof of Theorem~\ref{theorem:DPforDownCh}}\label{appendix:DPforDownCh}
To ensure a global $(\epsilon, \delta)$-DP in the uplink channels, the standard deviation of additive noises in client sides can be set to $\sigma_{\text{U}} = cL\Delta s_{\text{U}}/\epsilon$ due to the linear relation between $\epsilon$ and $\sigma_{\text{U}}$ with Gaussian mechanism, where $\Delta s_{\text{U}} = \frac{2C}{m}$ is the sensitivity for the aggregation operation and $m$ is the data size of each client.
We then set the sample in the $i$-th local noise vector to a same distribution $n_{i}\sim\varphi(n)$ (i.i.d for all $i$) because each client is coincident with the same global $(\epsilon, \delta)$-DP.
The aggregation process with artificial noises added by clients can be expressed as
\begin{equation}\label{equ:AppendA-1}
\begin{aligned}
\widetilde{\mathbf{w}} &= \sum\limits_{i=1}^{N}{p_{i}\left(\mathbf{w}_{i}+\mathbf{n}_{i}\right)}=\sum\limits_{i=1}^{N}{p_{i}\mathbf{w}_{i}}+\sum\limits_{i=1}^{N}{p_{i}{\mathbf{n}_{i}}}.
\end{aligned}
\end{equation}
The distribution $\phi_{N}(n)$ of $\sum_{i=1}^{N}{p_{i}{n_{i}}}$ can be expressed as
\begin{equation}\label{equ:AppendA-2}
\begin{aligned}
\phi_{N}(n)=\bigotimes\limits_{i=1}^{N}{\varphi_{i}\left(n\right)},
\end{aligned}
\end{equation}
where $p_{i}n_{i}\sim\varphi_{i}\left(n\right)$, and $\bigotimes$ is convolutional operation.

When we use Gaussian mechanism for $n_{i}$ with noise scale $\sigma_{\text{U}}$, the distribution of $p_{i}n_{i}$ is also Gaussian distribution. To obtain a small sensitivity $\Delta s_{\text{D}}$, we set $p_{i} = 1/N$.
Furthermore, the noise scale $\sigma_{\text{U}}/\sqrt{N}$ of the Gaussian distribution $\phi_{N}(n)$ can be calculated.
To ensure a global $(\epsilon, \delta)$-DP in downlink channels, we know the standard deviation of additive noises can be set to $\sigma_{\text{A}} = cT\Delta s_{\text{D}}/\epsilon$, where $\Delta s_{\text{D}} = 2C/mN$.
Hence, we can obtain the standard deviation of additive noises at the server as
\begin{multline}
\sigma_{\text{D}}=\sqrt{\sigma_{\text{A}}^{2}-\frac{\sigma_{\text{U}}^{2}}{N}}
=
\begin{cases}
\frac{2cC\sqrt{T^{2}-L^{2}N}}{mN\epsilon} & T>L\sqrt{N},\\
0&T\leq L\sqrt{N}.
\end{cases}
\end{multline}
Hence, \textbf{Theorem~\ref{theorem:DPforDownCh}} has been proved. $\hfill\square$
\section{Proof of Lemma~\ref{lem:Bdissimilarity}}\label{appendix:B_dissimilarity}
Due to \textbf{Assumption~\ref{ass:LossFunction}}, we have
\begin{equation}\label{equ:AppendB-1}
\begin{aligned}
\mathbb{E}& \left\{\Vert\nabla F_{i}(\mathbf{w})-\nabla F(\mathbf{w})\Vert^{2}\right\}\leq\mathbb{E}\{\varepsilon_{i}^{2}\}
\end{aligned}
\end{equation}
and
\begin{equation}\label{equ:AppendB-2}
\begin{aligned}
\mathbb{E}&\left\{\Vert\nabla F_{i}(\mathbf{w})-\nabla F(\mathbf{w})\Vert^{2}\right\}\\
&=\mathbb{E}\left\{ \Vert\nabla F_{i}(\mathbf{w})\Vert^{2}\right\}
-2\mathbb{E}\left\{\nabla F_{i}(\mathbf{w})^{\top}\right\}\nabla F(\mathbf{w})\\
&+\Vert\nabla F(\mathbf{w})\Vert^{2}
=\mathbb{E}\left\{ \Vert\nabla F_{i}(\mathbf{w})\Vert^{2}\right\}-\Vert\nabla F(\mathbf{w})\Vert^{2}.
\end{aligned}
\end{equation}
Considering~\eqref{equ:AppendB-1},~\eqref{equ:AppendB-2} and $\nabla F(\mathbf{w}) = \mathbb{E}\{\nabla F_{i}(\mathbf{w})\}$,
we have
\begin{equation}\label{equ:AppendB-3}
\begin{aligned}
\mathbb{E}\left\{ \Vert\nabla F_{i}(\mathbf{w})\Vert^{2}\right\}&\leq\Vert\nabla F(\mathbf{w})\Vert^{2}+ \mathbb{E}\{\varepsilon_{i}^{2}\}\\
&=\Vert\nabla F(\mathbf{w})\Vert^{2}B(\mathbf{w})^{2}.
\end{aligned}
\end{equation}
Note that when $\Vert\nabla F(\mathbf{w})\Vert^{2}\neq0$, there exists
\begin{equation}
B(\mathbf{w}) =\sqrt{1+\frac{\mathbb{E}\{\varepsilon_{i}^{2}\}}{\Vert\nabla F(\mathbf{w})\Vert^{2}}}\geq 1,
\end{equation}
which satisfies the equation.
We can notice that a smaller value of $B(\mathbf{w})$ implies that the local loss functions are more locally similar.
When all the local loss functions are the same, then $B(\mathbf{w}) = 1$, for all $\mathbf{w}$.
Therefore, we can have
\begin{equation}
\mathbb{E}\left\{ \Vert\nabla F_{i}(\mathbf{w})\Vert^{2}\right\}\leq\Vert\nabla F(\mathbf{w})\Vert^{2}B^{2},\quad{\forall i}, 
\end{equation}
where $B$ is the upper bound of $B(\mathbf{w})$.
This completes the proof.$\hfill\square$
\section{Proof of Lemma~\ref{theorem:Expincrement}}\label{appendix:ConvforN}
Considering the aggregation process with artificial noises added by clients and the server in the $(t+1)$-th aggregation, we have
\begin{equation}
\widetilde{\mathbf{w}}^{(t+1)}= \sum_{i=1}^{N}{p_{i}\mathbf{w}^{(t+1)}_{i}}+\mathbf{n}^{(t+1)},
\end{equation}
where
\begin{equation}
\mathbf{n}^{(t)} = \sum_{i=1}^{N}{p_{i}\mathbf{n}_{i}^{(t)}}+\mathbf{n}_{\text D}^{(t)}.
\end{equation}
Because $F_{i}(\cdot)$ is $\rho$-Lipschitz smooth, we know
\begin{multline}
F_{i}(\widetilde{\mathbf{w}}^{(t+1)})\leq F_{i}(\widetilde{\mathbf{w}}^{(t)})+\nabla F_{i}(\widetilde{\mathbf{w}}^{(t)})^{\top}(\widetilde{\mathbf{w}}^{(t+1)}-\widetilde{\mathbf{w}}^{(t)})\\
+\frac{\rho}{2}\Vert \widetilde{\mathbf{w}}^{(t+1)}-\widetilde{\mathbf{w}}^{(t)}\Vert^{2},
\end{multline}
for all $\widetilde{\mathbf{w}}^{(t+1)}$, $\widetilde{\mathbf{w}}^{(t)}$. Combining $F(\widetilde{\mathbf{w}}^{(t)})=\mathbb{E}\{F_{i}(\widetilde{\mathbf{w}}^{(t)})\}$ and $\nabla F(\widetilde{\mathbf{w}}^{(t)})=\mathbb{E}\{\nabla F_{i}(\widetilde{\mathbf{w}}^{(t)})\}$,
we have
\begin{multline}\label{equ:AppendC-1}
\mathbb{E}\{F(\widetilde{\mathbf{w}}^{(t+1)})-F(\widetilde{\mathbf{w}}^{(t)})\}
\leq\mathbb{E}\{\nabla F(\widetilde{\mathbf{w}}^{(t)})^{\top}(\widetilde{\mathbf{w}}^{(t+1)}-\widetilde{\mathbf{w}}^{(t)})\}\\
+\frac{\rho}{2}\mathbb{E}\{\Vert\widetilde{\mathbf{w}}^{(t+1)}-\widetilde{\mathbf{w}}^{(t)}\Vert^{2}\}.
\end{multline}
We define
\begin{multline}
J(\mathbf{w}^{(t+1)}_{i}; \widetilde{\mathbf{w}}^{(t)}) \triangleq F_{i}(\mathbf{w}^{(t+1)}_{i})+\frac{\mu}{2}\Vert\mathbf{w}^{(t+1)}_{i}-\widetilde{\mathbf{w}}^{(t)}\Vert^{2}.
\end{multline}
Then, we know
\begin{multline}
\nabla J(\mathbf{w}^{(t+1)}_{i}; \widetilde{\mathbf{w}}^{(t)}) = \nabla F_{i}(\mathbf{w}^{(t+1)}_{i})\\
+\mu\left(\mathbf{w}^{(t+1)}_{i}-\widetilde{\mathbf{w}}^{(t)}\right)
\end{multline}
and
\begin{multline}\label{equ:AppendC-2}
\widetilde{\mathbf{w}}^{(t+1)}-\widetilde{\mathbf{w}}^{(t)}
=\sum_{i=1}^{N} \left(\mathbf{w}^{(t+1)}_{i}+\mathbf{n}_{i}^{(t+1)}\right)+\mathbf{n}_{\text{D}}^{(t+1)}-\widetilde{\mathbf{w}}^{(t)}\\
=\frac{1}{\mu}\mathbb{E}\{\nabla J(\mathbf{w}^{(t+1)}_{i}; \widetilde{\mathbf{w}}^{(t)})-\nabla F_{i}(\mathbf{w}^{(t+1)}_{i})\}+\mathbf{n}^{(t+1)}.
\end{multline}
Because $F_{i}(\cdot)$ is $\rho$-Lipschitz smooth, we can obtain
\begin{multline}\label{equ:AppendC-3}
\mathbb{E}\{\nabla F_{i}(\mathbf{w}_{i}^{(t+1)})\} \leq \mathbb{E}\{\nabla F_{i}(\widetilde{\mathbf{w}}^{(t)})+\rho\Vert\mathbf{w}_{i}^{(t+1)}-\widetilde{\mathbf{w}}^{(t)}\Vert\}\\
=\nabla F(\widetilde{\mathbf{w}}^{(t)})+\rho\mathbb{E}\{\Vert\mathbf{w}_{i}^{(t+1)}-\widetilde{\mathbf{w}}^{(t)}\Vert\}.
\end{multline}
Now, let us bound $\Vert\mathbf{w}_{i}^{(t+1)}-\widetilde{\mathbf{w}}^{(t)}\Vert$.
We know
\begin{equation}\label{equ:AppendC-4}
\Vert\mathbf{w}_{i}^{(t+1)}-\widetilde{\mathbf{w}}^{(t)}\Vert \leq \Vert\mathbf{w}^{(t+1)}_{i}-\widehat{\mathbf{w}}_{i}^{(t+1)}\Vert + \Vert\widehat{\mathbf{w}}_{i}^{(t+1)}-\widetilde{\mathbf{w}}^{(t)} \Vert,
\end{equation}
where $\widehat{\mathbf{w}}_{i}^{(t+1)}=\arg\min_{\mathbf{w}}{J_{i}(\mathbf{w};\widetilde{\mathbf{w}}^{(t)})}$.
Let us define $\overline{\mu}=\mu-\rho>0$, then we know $J_{i}(\mathbf{w};\widetilde{\mathbf{w}}^{(t)})$ is $\overline{\mu}$-convexity.
Based on this, we can obtain
\begin{equation}\label{equ:AppendC-5}
\Vert\widehat{\mathbf{w}}_{i}^{(t+1)}-\mathbf{w}^{(t+1)}_{i} \Vert\leq \frac{\theta}{\overline{\mu}}\Vert\nabla F_{i}(\widetilde{\mathbf{w}}^{(t)})\Vert
\end{equation}
and
\begin{equation}\label{equ:AppendC-6}
\Vert\widehat{\mathbf{w}}_{i}^{(t+1)}-\widetilde{\mathbf{w}}^{(t)} \Vert\leq \frac{1}{\overline{\mu}}\Vert\nabla F_{i}(\widetilde{\mathbf{w}}^{(t)})\Vert,
\end{equation}
where $\theta$ denotes a $\theta$ solution of $\min_{\mathbf{w}}J_{i}(\mathbf{w};\widetilde{\mathbf{w}}^{(t)})$~\cite{DBLP:journals/corr/abs-1812-06127}.
Now, we can use the inequality \eqref{equ:AppendC-5} and \eqref{equ:AppendC-6} to obtain
\begin{equation}\label{equ:AppendC-7}
\Vert\mathbf{w}^{(t+1)}_{i}-\widetilde{\mathbf{w}}^{(t)} \Vert\leq \frac{1+\theta}{\overline{\mu}}\Vert\nabla F_{i}(\widetilde{\mathbf{w}}^{(t)})\Vert.
\end{equation}
Therefore,
\begin{equation}\label{equ:AppendC-8}
\begin{aligned}
&\Vert\widetilde{\mathbf{w}}^{(t+1)}-\widetilde{\mathbf{w}}^{(t)} \Vert
\leq \Vert\mathbf{w}^{(t+1)}-\widetilde{\mathbf{w}}^{(t)}\Vert+\Vert\mathbf{n}^{(t+1)}\Vert\\
&\quad\quad\leq\mathbb{E}\{\Vert\mathbf{w}_{i}^{(t+1)}-\widetilde{\mathbf{w}}^{(t)}\Vert\}+\Vert\mathbf{n}^{(t+1)}\Vert\\
&\quad\quad\leq \frac{1+\theta}{\overline{\mu}}\mathbb{E}\{\Vert\nabla F_{i}(\widetilde{\mathbf{w}}^{(t)})\Vert\}+\Vert\mathbf{n}^{(t+1)}\Vert\\
&\quad\quad\leq \frac{B(1+\theta)}{\overline{\mu}}\Vert\nabla F(\widetilde{\mathbf{w}}^{(t)})\Vert+\Vert\mathbf{n}^{(t+1)}\Vert.
\end{aligned}
\end{equation}
Using~\eqref{equ:AppendC-3} and~\eqref{equ:AppendC-4}, we know
\begin{equation}\label{equ:AppendC-9}
\begin{aligned}
\Vert\mathbb{E}&\{\nabla F_{i}(\mathbf{w}_{i}^{(t+1)})\}-\nabla F(\widetilde{\mathbf{w}}^{(t)})-\mathbb{E}\{\nabla J(\mathbf{w}^{(t+1)}_{i}; \widetilde{\mathbf{w}}^{(t)})\}\Vert\\
&\leq \rho\mathbb{E}\{\Vert\mathbf{w}_{i}^{(t+1)}-\widetilde{\mathbf{w}}^{(t)}\Vert\}+\mathbb{E}\{\nabla J(\mathbf{w}^{(t+1)}_{i}; \widetilde{\mathbf{w}}^{(t)})\}\\
&\leq\frac{ \rho B(1+\theta)}{\overline{\mu}}\Vert\nabla F(\widetilde{\mathbf{w}}^{(t)})\Vert+ B\theta\Vert\nabla F(\widetilde{\mathbf{w}}^{(t)})\Vert.\\
\end{aligned}
\end{equation}
Substituting~\eqref{equ:AppendC-3},~\eqref{equ:AppendC-8} and~\eqref{equ:AppendC-9} into~\eqref{equ:AppendC-1}, we know
\begin{multline}\label{equ:AppendC-10}
\mathbb{E}\{F(\widetilde{\mathbf{w}}^{(t+1)})-F(\widetilde{\mathbf{w}}^{(t)})\}\\
\leq\mathbb{E}\left\{\nabla F(\widetilde{\mathbf{w}}^{(t)})^{\top}\left(-\frac{1}{\mu}\nabla F(\widetilde{\mathbf{w}}^{(t)})+\frac{1}{\mu}\mathbf{n}^{(t+1)}\right.\right.\\
\quad\left.\left.+\left(\frac{ \rho B(1+\theta)}{\mu\overline{\mu}}+\frac{ B\theta}{\mu}\right)\Vert\nabla F(\widetilde{\mathbf{w}}^{(t)})\Vert\right)\right\}\\
\quad+\frac{\rho}{2}\mathbb{E}\left\{\left[\frac{B(1+\theta)}{\overline{\mu}}\Vert\nabla F(\widetilde{\mathbf{w}}^{(t)})\Vert+\Vert\mathbf{n}^{(t+1)}\Vert\right]^{2}\right\}.
\end{multline}
Then, using triangle inequation, we can obtain
\begin{equation}\label{equ:AppendC-11}
\begin{aligned}
\mathbb{E}&\{F(\widetilde{\mathbf{w}}^{(t+1)})-F(\widetilde{\mathbf{w}}^{(t)})\}\leq\lambda_{2}\Vert\nabla F(\widetilde{\mathbf{w}}^{(t)})\Vert^{2}\\
&+\lambda_{1}\mathbb{E}\{\Vert\mathbf{n}^{(t+1)}\Vert\}\Vert\nabla F(\widetilde{\mathbf{w}}^{(t)})\Vert+\lambda_{0}\mathbb{E}\{\Vert\mathbf{n}^{(t+1)}\Vert^{2}\}.
\end{aligned}
\end{equation}
where
\begin{equation}
\lambda_{2}=-\frac{1}{\mu}+\frac{B}{\mu}\left[\frac{\rho(1+\theta)}{\overline{\mu}}+\theta\right]+\frac{\rho B^{2}(1+\theta)^{2}}{2\overline{\mu}^{2}},
\end{equation}
\begin{equation}
\lambda_{1} = \frac{1}{\mu}+\frac{\rho B(1+\theta)}{\overline{\mu}} \,\text{and}\, \lambda_{0} = \frac{\rho}{2}.
\end{equation}
In this convex case, where $\overline{\mu}=\mu$, if $\theta = 0$, all subproblems are solved accurately. We know $\lambda_{2}=-\frac{1}{\mu}+\frac{\rho B}{\mu^{2}}+\frac{\rho B^{2}}{2\mu^{2}}$,
$\lambda_{1} = \frac{1}{\mu}+\frac{\rho B}{\mu}$
and $\lambda_{0} = \frac{\rho}{2}$.
This completes the proof. $\hfill\square$
\section{Proof of Theorem~\ref{theorem:ConvUpperBound}}\label{appendix:ConvUpperBound}
We assume that $F$ satisfies the Polyak-Lojasiewicz inequality~\cite{Nesterov:2014:ILC:2670022} with positive parameter $l$, which implies that
\begin{equation}\label{equ:AppendD-1}
\begin{aligned}
\mathbb{E}\{F(\widetilde{\mathbf{w}}^{(t)})-F(\mathbf{w}^{*})\}\leq\frac{1}{2l}\Vert \nabla F(\widetilde{\mathbf{w}}^{(t)})\Vert^{2}.
\end{aligned}
\end{equation}
Moreover, subtract $\mathbb{E}\{F(\mathbf{w}^{*})\}$ in both sides of \eqref{equ:AppendC-11}, we know
\begin{multline}
\mathbb{E}\{F(\widetilde{\mathbf{w}}^{(t+1)})-F(\mathbf{w}^{*})\}\\
\leq \mathbb{E}\{F(\widetilde{\mathbf{w}}^{(t)})-F(\mathbf{w}^{*})\}
+\lambda_{2}\Vert\nabla F(\widetilde{\mathbf{w}}^{(t)})\Vert^{2}\\
+\lambda_{1}\mathbb{E}\{\Vert\mathbf{n}^{(t+1)}\Vert\}\Vert\nabla F(\widetilde{\mathbf{w}}^{(t)})\Vert
+\lambda_{0}\mathbb{E}\{\Vert\mathbf{n}^{(t+1)}\Vert^{2}\}.
\end{multline}
Considering $\Vert\nabla F(\mathbf{w}^{(t)})\Vert\leq\beta$ and~\eqref{equ:AppendD-1}, we have
\begin{multline}\label{equ:AppendD-2}
\mathbb{E}\{F(\widetilde{\mathbf{w}}^{(t+1)})-F(\mathbf{w}^{*})\}\leq (1+2l\lambda_{2})\mathbb{E}\{F(\widetilde{\mathbf{w}}^{(t)})-F(\mathbf{w}^{*})\}\\
+\lambda_{1}\beta\mathbb{E}\{\Vert\mathbf{n}^{(t+1)}\Vert\}+\lambda_{0}\mathbb{E}\{\Vert\mathbf{n}^{(t+1)}\Vert^{2}\},
\end{multline}
where $F(\mathbf{w}^{*})$ is the loss function corresponding to the optimal parameters $\mathbf{w}^{*}$.
Considering the same and independent distribution of additive noises, we define $\mathbb{E}\{\Vert\mathbf{n}^{(t)}\Vert\}=\mathbb{E}\{\Vert\mathbf{n}\Vert\}$ and $\mathbb{E}\{\Vert\mathbf{n}^{(t)}\Vert^{2}\}=\mathbb{E}\{\Vert\mathbf{n}\Vert^{2}\}$, for $0 \leq t \leq T$.
Applying \eqref{equ:AppendD-2} recursively, we have
\begin{multline}\label{equ:AppendD-3}
\mathbb{E}\{F(\widetilde{\mathbf{w}}^{(T)})-F(\mathbf{w}^{*})\}
\leq (1+2l\lambda_{2})^{T}\mathbb{E}\{F(\mathbf{w}^{(0)})-F(\mathbf{w}^{*})\}\\
+\left(\lambda_{1}\beta\mathbb{E}\{\Vert\mathbf{n}\Vert\}+\lambda_{0}\mathbb{E}\{\Vert\mathbf{n}\Vert^{2}\}\right)\sum_{t=0}^{T-1}{\left(1+2l\lambda_{2}\right)^t}\\
=(1+2l\lambda_{2})^{T}\mathbb{E}\{F(\mathbf{w}^{(0)})-F(\mathbf{w}^{*})\}\\
+\left(\lambda_{1}\beta\mathbb{E}\{\Vert\mathbf{n}\Vert\}+\lambda_{0}\mathbb{E}\{\Vert\mathbf{n}\Vert^{2}\}\right)\frac{\left(1+2l\lambda_{2}\right)^T-1}{2l\lambda_{2}}.
\end{multline}

If $T\leq L\sqrt{N}$ and then $\sigma_{\text{D}}= 0$, this case is special. Hence, we will consider the condition that $T>L\sqrt{N}$.
Based on~\eqref{equ:theorem2-2}, we have $\sigma_\text{A} = \Delta s_{\text{D}}Tc/\epsilon$.
Hence, we can obtain
\begin{multline}\label{equ:AppendD-4}
\mathbb{E}\{\Vert \mathbf{n}\Vert\} = \frac{\Delta s_{\text{D}}Tc}{\epsilon}\sqrt{\frac{2N}{\pi }}\,\text{and}\, \mathbb{E}\{\Vert \mathbf{n}\Vert^{2}\} = \frac{\Delta s_{\text{D}}^{2}T^{2}c^{2}N}{\epsilon^{2}}.
\end{multline}
Substituting \eqref{equ:AppendD-4} into \eqref{equ:AppendD-3}, setting $\Delta s_{\text{D}} = 1/N$ and $F(\mathbf{w}^{(0)})-F(\mathbf{w}^{*})=\Theta$, we have
\begin{equation}\label{equ:AppendD-6}
\begin{aligned}
&\mathbb{E}\{F(\widetilde{\mathbf{w}}^{(T)})-F(\mathbf{w}^{*})\}
\leq (1+2l\lambda_{2})^{T}\Theta\\
&\quad+\left(\frac{\lambda_{1}T\beta c}{\epsilon}\sqrt{\frac{2}{N\pi}}+\frac{\lambda_{0}T^{2}c^{2}}{\epsilon^{2}N}\right)\frac{\left(1+2l\lambda_{2}\right)^T-1}{2l\lambda_{2}}\\
&=P^{T}\Theta+\left(\frac{\kappa_{1}T}{\epsilon}+\frac{\kappa_{0}T^{2}}{\epsilon^{2}}\right)\left(1-P^T\right),
\end{aligned}
\end{equation}
where $P = 1+2l\lambda_{2}$, $\kappa_{1} = \frac{\lambda_{1}\beta c}{m(P-1)}\sqrt{\frac{2}{N\pi}}$ and $\kappa_{0} = \frac{\lambda_{0}c^{2}}{m^2(P-1)N}$.
This completes the proof. $\hfill\square$
\section{Proof of Lemma~\ref{lem:NoiseKSche}}\label{appendix:DPforKshe}
We define the sampling parameter $q\triangleq K/N$ to  represent the probability of being selected by the server for each client in an aggregation.
Let $\mathcal{M}_{1:T}$ denote $(\mathcal{M}_{1},\ldots,\mathcal{M}_{T})$ and similarly let $o_{1:T}$ denote a sequence of outcomes $(o_{1},\ldots,o_{T})$.
Considering a global $(\epsilon, \delta)$-DP in the downlinks channels, we use $\sigma_{\text{A}}$ to represent the standard deviation of aggregated Gaussian noises. With neighboring datasets $\mathcal D_i$ and $\mathcal D'_i$, we are looking at
\begin{equation}\label{equ:AppendF-1}
\begin{aligned}
&\left\vert\ln\frac{\Pr[\mathcal{M}_{1:T}(\mathcal D'_{i,1:T})=o_{1:T}]}{\Pr[\mathcal{M}_{1:T}(\mathcal D_{i,1:T})=o_{1:T}]}\right\vert\\
&\quad= \left\vert\sum_{i=1}^{T}{\ln\frac{(1-q)e^{-\frac{\Vert
n\Vert^{2}}{2\sigma_{\text{A}}^{2}}}+qe^{-\frac{\Vert
n+\Delta s_{\text{D}}\Vert^{2}}{2\sigma_{\text{A}}^{2}}}}{e^{-\frac{\Vert n\Vert^{2}}{2\sigma_{\text{A}}^{2}}}}}\right\vert\\
&\quad=\left\vert\sum_{i=1}^{T}{\ln\left(1-q+qe^{-\frac{2n\Delta s_{\text{D}}+
\Delta s_{\text{D}}^{2}}{2\sigma_{\text{A}}^{2}}}\right)}\right\vert\\
&\quad=\left\vert{\ln\prod\limits_{i=1}^{T}\left(1-q+qe^{-\frac{2n\Delta s_{\text{D}}+
\Delta s_{\text{D}}^{2}}{2\sigma_{\text{A}}^{2}}}\right)}\right\vert.
\end{aligned}
\end{equation}
This quantity is bounded by $\epsilon$, we require
\begin{equation}\label{equ:AppendF-2}
\begin{aligned}
\left\vert\ln\frac{\Pr[\mathcal{M}_{1:T}(\mathcal D'_{i,1:T})=o_{1:T}]}{\Pr[\mathcal{M}_{1:T}(\mathcal D_{i,1:T})=o_{1:T}]}\right\vert
\leq\epsilon.
\end{aligned}
\end{equation}
Considering the independence of adding noises, we know
\begin{equation}\label{equ:AppendF-3}
\begin{aligned}
T\ln\left(1-q+qe^{-\frac{2n\Delta s_{\text{D}}+\Vert\Delta s_{\text{D}}\Vert^{2}}{2\sigma_{\text{A}}^{2}}}\right)
\geq-\epsilon.
\end{aligned}
\end{equation}
We can obtain the result
\begin{equation}
n\leq -\frac{\sigma_{\text{A}}^{2}}{\Delta s_{\text{D}}}\ln\left(\frac{\exp(-\frac{\epsilon}{T})}{q}-\frac{1}{q}+1\right)-\frac{\Delta s_{\text{D}}}{2}.
\end{equation}
We set
\begin{equation}\label{equ:AppendF-5}
\begin{aligned}
b=-\frac{T}{\epsilon}\ln\left(\frac{\exp(-\epsilon/T)-1}{q}+1\right).
\end{aligned}
\end{equation}
Hence,
\begin{equation}\label{equ:AppendF-4}
\begin{aligned}
\ln\left(\frac{\exp(-\epsilon/T)-1}{q}+1\right)=-\frac{b\epsilon}{T}.
\end{aligned}
\end{equation}
Note that $\epsilon$ and $T$ should satisfy
\begin{equation}\label{equ:AppendF-6}
\begin{aligned}
\epsilon<-T\ln\left(1-q\right) \,\,\text{or}\,\, T>\frac{-\epsilon}{\ln\left(1-q\right)}.
\end{aligned}
\end{equation}
Then,
\begin{equation}\label{equ:AppendF-7}
\begin{aligned}
n \leq \frac{\sigma_{\text{A}}^{2}b\epsilon}{T\Delta s_{\text{D}}}-\frac{\Delta s_{\text{D}}}{2}.
\end{aligned}
\end{equation}
Using the tail bound $\Pr[n>\eta]\leq\frac{\sigma_{\text{A}}}{\sqrt{2\pi}}\frac{1}{\eta}e^{-\eta^{2}/2\sigma_{\text{A}}^{2}}$,
we can obtain
\begin{equation}\label{equ:AppendF-8}
\begin{aligned}
\ln\left(\frac{\eta}{\sigma_{\text{A}}}\right) + \frac{\eta^{2}}{2\sigma_{\text{A}}^{2}} > \ln\left(\sqrt{\frac{2}{\pi}}\frac{1}{\delta}\right).
\end{aligned}
\end{equation}
Let us set $\sigma_{\text{A}} = c\Delta s_{\text{D}}T/b\epsilon$, if $b\epsilon/T \in(0, 1)$, the inequation~\eqref{equ:AppendF-8} can be solved as
\begin{equation}\label{equ:AppendF-9}
\begin{aligned}
c^{2}\geq 2\ln\left(\frac{1.25}{\delta}\right).
\end{aligned}
\end{equation}
Meanwhile, $\epsilon$ and $T$ should satisfy
\begin{equation}\label{equ:AppendF-10}
\begin{aligned}
\epsilon<-T\ln\left(1-q+\frac{q}{e}\right) \,\,\text{or}\,\, T>\frac{-\epsilon}{\ln\left(1-q+\frac{q}{e}\right)}.
\end{aligned}
\end{equation}
If $b\epsilon/T > 1$, we can also obtain $\sigma_{\text{A}} = c\Delta s_{\text{D}}T/b\epsilon$ by adjusting the value of $c$.
The standard deviation of requiring noises is given as
\begin{equation}\label{equ:AppendF-11}
\begin{aligned}
\sigma_{\text{A}} \geq \frac{c\Delta s_{\text{D}}T}{b\epsilon}.
\end{aligned}
\end{equation}
Hence, if Gaussian noises are added at the client sides, we can obtain the additive noise scale in the server as
\begin{multline}
\sigma_{\text{D}}=\sqrt{\left(\frac{c\Delta s_{\text{D}}T}{b\epsilon}\right)^{2}-\frac{c^{2}L^{2}\Delta s_{\text{U}}^{2}}{K\epsilon^{2}}}\\
=
\begin{cases}
\frac{2cC\sqrt{\frac{T^{2}}{b^{2}}-L^{2}K}}{mK\epsilon} & T>bL\sqrt{K},\\
0&T\leq bL\sqrt{K}.
\end{cases}
\end{multline}
Furthermore, considering~\eqref{equ:AppendF-6}, we can obtain
\begin{equation}
\begin{aligned}
\sigma_{\text{D}}=
\begin{cases}
\frac{2cC\sqrt{\frac{T^{2}}{b^{2}}-L^{2}K}}{mK\epsilon} & T>\frac{\epsilon}{\gamma},\\
0&T\leq \frac{\epsilon}{\gamma},
\end{cases}
\end{aligned}
\end{equation}
where
\begin{equation}
\begin{split}
\gamma = -\ln\left({1-q+ qe^{\frac{-\epsilon}{L\sqrt{K}}}}\right).
\end{split}
\end{equation}
This completes the proof. $\hfill\square$
\section{proof of Theorem~\ref{theorem:ConvKSche}}\label{appendix:ConvforK}
Here we define
\begin{equation}
\mathbf{v}^{(t)}=\sum_{i=1}^{K}{p_{i}\mathbf{w}^{(t)}_{i}},
\end{equation}
\begin{equation}
\widetilde{\mathbf{v}}^{(t)}=\sum_{i=1}^{K}{p_{i}\left(\mathbf{w}^{(t)}_{i}+\mathbf{n}^{(t)}_{i}\right)}+\mathbf{n}^{(t)}_{\text{D}}
\end{equation}
and
\begin{equation}
\mathbf{n}^{(t+1)}=\sum_{i=1}^{K}{p_{i}\mathbf{n}^{(t+1)}_{i}}+\mathbf{n}^{(t)}_{\text{D}}.
\end{equation}
which considers the aggregated parameters under $K$-random scheduling.
Because $F_{i}(\cdot)$ and $F(\cdot)$ are $\beta$-Lipschitz, we obtain that
\begin{equation}\label{equ:AppendE-1}
\begin{aligned}
\mathbb{E}\{F(\widetilde{\mathbf{v}}^{(t+1)})\}-F(\mathbf{w}^{(t+1)})\leq \beta\Vert\widetilde{\mathbf{v}}^{(t+1)}-\mathbf{w}^{(t+1)}\Vert.
\end{aligned}
\end{equation}
Because $\beta$ is the Lipchitz continuity constant of function $F$, we have
\begin{equation}\label{equ:AppendE-2}
\begin{aligned}
\beta \leq \Vert\nabla F(\widetilde{\mathbf{v}}^{(t)})\Vert+\rho\left(\Vert\mathbf{w}^{(t+1)}-\widetilde{\mathbf{v}}^{(t)}\Vert\right.\\
\left.+\Vert\mathbf{v}^{(t+1)}-\widetilde{\mathbf{v}}^{(t)}\Vert\right).
\end{aligned}
\end{equation}
From \eqref{equ:AppendC-8}, we know
\begin{equation}
\Vert\mathbf{w}^{(t+1)}-\widetilde{\mathbf{v}}^{(t)}\Vert\leq \frac{B(1+\theta)}{\overline{\mu}}\Vert\nabla F(\widetilde{\mathbf{v}}^{(t)})\Vert.
\end{equation}
Then, we have
\begin{multline}\label{equ:AppendE-3}
\mathbb{E}\{\Vert\mathbf{w}^{(t+1)}-\widetilde{\mathbf{v}}^{(t+1)}\Vert^{2}\}=\Vert\mathbf{w}^{(t+1)}\Vert^{2}\\
-2[\mathbf{w}^{(t+1)}]^{\top} \mathbb{E}\{\widetilde{\mathbf{v}}^{(t+1)}\}+\mathbb{E}\{\Vert\widetilde{\mathbf{v}}^{(t+1)}\Vert^{2}\}.
\end{multline}
Furthermore, we can obtain
\begin{equation}\label{equ:AppendE-4}
\begin{aligned}
\mathbb{E}&\{\widetilde{\mathbf{v}}^{(t+1)}\}=\frac{1}{\left(\begin{matrix}
   N  \\
   K \\
  \end{matrix}\right)}\frac{\left(\begin{matrix}
   N  \\
   K \\
  \end{matrix}\right)}{N}K\sum_{i=1}^{N}{p_{i}\mathbf{w}^{(t+1)}_{i}}+\mathbf{n}^{(t+1)}\\
&=\mathbb{E}\{\mathbf{w}_{i}^{(t+1)}\}+\mathbf{n}^{(t+1)}=\mathbf{w}^{(t+1)}+\mathbf{n}^{(t+1)}
\end{aligned}
\end{equation}
and
\begin{multline}\label{equ:AppendE-5}
\mathbb{E}\{\Vert\widetilde{\mathbf{v}}^{(t+1)}\Vert^{2}\}=\mathbb{E}\left\{\left\Vert \sum\limits_{i=1}^{K}{\left(p_{i}\mathbf{w}^{(t+1)}_{i}+p_{i}\mathbf{n}^{(t+1)}_{i}\right)}\right\Vert^{2}\right\}\\
=\mathbb{E}\left\{\left\Vert \sum\limits_{i=1}^{K}{p_{i}\mathbf{w}^{(t+1)}_{i}}\right\Vert^{2}\right\}+\mathbb{E}\left\{\left\Vert \sum\limits_{i=1}^{K}{p_{i}\mathbf{n}^{(t+1)}_{i}}\right\Vert^{2}\right\}\\
\quad+2\mathbb{E}\left\{\left[\sum\limits_{i=1}^{K}{p_{i}\mathbf{w}^{(t+1)}_{i}}\right]^{\top}\mathbf{n}^{(t+1)}\right\}.\\
\end{multline}
Due the independence between $\mathbf{w}^{(t+1)}_{i}$ and $\mathbf{n}^{(t+1)}_{i}$, we know
\begin{multline}\label{equ:AppendE-6}
\mathbb{E}\left\{\left\Vert \sum\limits_{i=1}^{K}{p_{i}\mathbf{w}^{(t+1)}_{i}}\right\Vert^{2}\right\}=\mathbb{E}\left\{ \sum\limits_{i=1}^{K}{\left\Vert p_{i}\mathbf{w}^{(t+1)}_{i}\right\Vert^{2}}\right\}\\.
\end{multline}
Note that we set $p_{i} = D_i/\sum_{i=1}^{K}{D_{i}}=1/K$ in $K$-random scheduling in order to a small sensitivity $\Delta s_{\text{D}}$.
We have
\begin{multline}
\mathbb{E}\left\{\left\Vert \sum\limits_{i=1}^{K}{p_{i}\mathbf{w}^{(t+1)}_{i}}\right\Vert^{2}\right\}=\frac{1}{NK} \sum\limits_{i=1}^{N}{\left\Vert\mathbf{w}^{(t+1)}_{i}\right\Vert^{2}}\\
+\frac{(K-1)}{NK(N-1)}\sum\limits_{i=1}^{N}\sum\limits_{j=1\bigcup j\neq i}^{N}{\left[ \mathbf{w}^{(t+1)}_{i}\right]^{\top}\mathbf{w}^{(t+1)}_{j}}\\
\leq \frac{1}{K^{2}}\sum\limits_{i = 1}^{K}\Vert\mathbf{w}^{(t+1)}_{i}\Vert^{2}+\frac{K-1}{K}\Vert\mathbf{w}^{(t+1)}\Vert^{2}
\end{multline}
and
\begin{multline}\label{equ:AppendE-7}
\mathbb{E}\{\Vert\widetilde{\mathbf{v}}^{(t+1)}\Vert^{2}\}\leq \frac{1}{K^{2}}\sum_{i=1}^{K}\Vert\mathbf{w}^{(t+1)}_{i}\Vert^{2}+\frac{K-1}{K}\Vert\mathbf{w}^{(t+1)}\Vert^{2}\\
\quad+\Vert \mathbf{n}^{(t+1)}\Vert^{2} + 2[\mathbf{w}^{(t+1)}]^{\top}\mathbf{n}^{(t+1)}.
\end{multline}
Combining \eqref{equ:AppendE-3} and \eqref{equ:AppendE-7}, we can obtain
\begin{multline}\label{equ:AppendE-8}
\mathbb{E}\{\Vert\mathbf{w}^{(t+1)}-\widetilde{\mathbf{v}}^{(t+1)}\Vert^{2}\}\\
\leq\frac{1}{K^{2}}\sum_{i=1}^{K}\Vert\mathbf{w}^{(t+1)}_{i}\Vert^{2}-\frac{1}{K}\Vert\mathbf{w}^{(t+1)}\Vert^{2}+\Vert\mathbf{n}^{(t+1)}\Vert^{2}\\
\leq\frac{1}{K^{2}}\sum_{i=1}^{K}\Vert\mathbf{w}^{(t+1)}_{i}-\widetilde{\mathbf{v}}^{(t)}\Vert^{2}+\Vert\mathbf{n}^{(t+1)}\Vert^{2}.\\
\end{multline}
Using \eqref{equ:AppendC-7}, we know
\begin{multline}\label{equ:AppendE-9}
\mathbb{E}\{\Vert\mathbf{w}^{(t+1)}-\widetilde{\mathbf{v}}^{(t+1)}\Vert^{2}\}\leq\Vert\mathbf{n}^{(t+1)}\Vert^{2}+\\\frac{B^{2}(1+\theta)^{2}}{K\overline{\mu}^{2}}\Vert\nabla F(\widetilde{\mathbf{v}}^{(t)})\Vert^{2}.
\end{multline}
Moreover,
\begin{multline}\label{equ:AppendE-10}
\mathbb{E}\{\Vert\mathbf{w}^{(t+1)}-\widetilde{\mathbf{v}}^{(t+1)}\Vert\}\leq\Vert\mathbf{n}^{(t+1)}\Vert\\
+\frac{B(1+\theta)}{\overline{\mu}\sqrt{K}}\Vert\nabla F(\widetilde{\mathbf{v}}^{t})\Vert.
\end{multline}
Substituting~\eqref{equ:AppendC-11},~\eqref{equ:AppendE-2} and~\eqref{equ:AppendE-10} into~\eqref{equ:AppendE-1}, setting $\theta = 0$ and $\overline{\mu} = \mu$, we can obtain
\begin{multline}\label{equ:AppendE-11}
\mathbb{E}\{F(\widetilde{\mathbf{v}}^{(t+1)})\}-F(\widetilde{\mathbf{v}}^{(t)})\leq F(\mathbf{w}^{(t+1)})-F(\widetilde{\mathbf{v}}^{(t)})\\
\left(\Vert\nabla F(\widetilde{\mathbf{v}}^{(t)})\Vert+2\rho\Vert\mathbf{w}^{(t+1)}-\widetilde{\mathbf{v}}^{(t)}\Vert\right)\mathbb{E}\Vert\mathbf{w}^{(t+1)}-\widetilde{\mathbf{v}}^{(t+1)}\Vert\\
+\rho\mathbb{E}\{\Vert\mathbf{w}^{(t+1)}-\widetilde{\mathbf{v}}^{(t+1)}\Vert^{2}\}=\alpha_{2}\Vert\nabla F(\widetilde{\mathbf{v}}^{(t)})\Vert^{2}\\
+\alpha_{1} \Vert\mathbf{n}^{(t+1)}\Vert\Vert\nabla F(\widetilde{\mathbf{v}}^{(t)})\Vert+\alpha_{0}\Vert
\mathbf{n}^{(t+1)}\Vert^{2},
\end{multline}
where
\begin{equation}
\alpha_{2}=\frac{1}{\mu^{2}}\left(\frac{\rho B^{2}}{2}+\rho B+\frac{\rho B^{2}}{K}+\frac{2\rho B^{2}}{\sqrt{K}}+\frac{\mu B}{\sqrt{K}}-\mu\right),
\end{equation}
\begin{multline}
\alpha_{1} = 1+\frac{2\rho B}{\mu}+\frac{2\rho B\sqrt{K}}{\mu N} \,\text{and}\, \alpha_{0} = \frac{2\rho K}{N}+\rho.
\end{multline}

In this case, we take expectation $\mathbb{E}\{F(\widetilde{\mathbf{v}}^{(t+1)})-F(\widetilde{\mathbf{v}}^{(t)})\}$ as follows,
\begin{equation}\label{equ:AppendE-12}
\begin{aligned}
\mathbb{E}&\{F(\widetilde{\mathbf{v}}^{(t+1)})-F(\widetilde{\mathbf{v}}^{(t)})\}\leq \alpha_{2}\Vert\nabla F(\widetilde{\mathbf{v}}^{(t)})\Vert^{2}\\
&+\alpha_{1} \mathbb{E}\{\Vert\mathbf{n}^{(t+1)}\Vert\}\Vert\nabla F(\widetilde{\mathbf{v}}^{(t)})\Vert+\alpha_{0}\mathbb{E}\{\Vert\mathbf{n}^{(t+1)}\Vert^{2}\}.
\end{aligned}
\end{equation}
For $\Theta>0$ and $f(\mathbf{v}^{(0)})-f(\mathbf{w}^{*})=\Theta$, we can obtain
\begin{multline}\label{equ:AppendE-13}
\mathbb{E}\{F(\widetilde{\mathbf{v}}^{(t+1)})-F(\mathbf{w}^{*})\}\\
\leq\mathbb{E}\{F(\widetilde{\mathbf{v}}^{(t)})-F(\mathbf{w}^{*})\}+
\alpha_{2}\Vert\nabla F(\widetilde{\mathbf{v}}^{(t)})\Vert^{2}\\
+\alpha_{1}\beta\mathbb{E}\{\Vert\mathbf{n}^{(t+1)}\Vert\}+ \alpha_{0}\mathbb{E}\{\Vert\mathbf{n}^{(t+1)}\Vert^{2}\}.
\end{multline}
If we select the penalty parameter $\mu$ to make $\alpha_{2}<0$ and using \eqref{equ:AppendD-1}, we know
\begin{multline}\label{equ:AppendE-14}
\mathbb{E}\{F(\widetilde{\mathbf{v}}^{(t+1)})-F(\mathbf{w}^{*})\}\leq(1+2l\alpha_{2})\mathbb{E}\{F(\widetilde{\mathbf{v}}^{(t)})-F(\mathbf{w}^{*})\}\\
+\alpha_{1}\beta\mathbb{E}\{\Vert\mathbf{n}^{(t+1)}\Vert\}+ \alpha_{0}\mathbb{E}\{\Vert\mathbf{n}^{(t+1)}\Vert^{2}\}.
\end{multline}
Considering independence of additive noises and applying~\eqref{equ:AppendE-14} recursively, we have
\begin{multline}\label{equ:AppendE-15}
\mathbb{E}\{F(\widetilde{\mathbf{v}}^{(T)})-F(\mathbf{w}^{*})\}\leq(1+2l\alpha_{2})^{T}\mathbb{E}\{F(\mathbf{v}^{(0)})-F(\mathbf{w}^{*})\}\\
+\frac{1-(1+2l\alpha_{2})^{T}}{2l\alpha_{2}}\left(\alpha_{1}\beta\mathbb{E}\{\Vert\mathbf{n}\Vert\}+ \alpha_{0}\mathbb{E}\{\Vert\mathbf{n}\Vert^{2}\}\right)\\
=Q^{T}\Theta+\frac{1-Q^{T}}{1-Q}\left(\alpha_{1}\beta\mathbb{E}\{\Vert\mathbf{n}\Vert\}+ \alpha_{0}\mathbb{E}\{\Vert\mathbf{n}\Vert^{2}\}\right),
\end{multline}
where $Q = 1+2l\alpha_{2}$.
Substituting~\eqref{equ:AppendF-11} into~\eqref{equ:AppendE-15}, we can obtain
\begin{multline}\label{equ:AppendD-4}
\mathbb{E}\{\Vert \mathbf{n}\Vert\} = \frac{\Delta s_{\text{D}}Tc}{b\epsilon}\sqrt{\frac{2N}{\pi }}, \mathbb{E}\{\Vert \mathbf{n}\Vert^{2}\} = \frac{\Delta s_{\text{D}}^{2}T^{2}c^{2}N}{b^{2}\epsilon^{2}}
\end{multline}
and
\begin{equation}
\begin{aligned}
&\mathbb{E}\{F(\widetilde{\mathbf{v}}^{T})-F(\mathbf{w}^{*})\}
\leq Q^{T}\Theta \\
&\quad+ \frac{1-Q^{T}}{1-Q}\left(\frac{c\alpha_{1}\beta}{-mK\ln\left(1-\frac{N}{K}+\frac{N}{K}e^{-\frac{\epsilon}{T}}\right)}\sqrt{\frac{2}{\pi}}\right.\\
&\quad+\left.\frac{c^{2}\alpha_{0}}{m^{2}K^{2}\ln^{2}\left(1-\frac{N}{K}+\frac{N}{K}e^{-\frac{\epsilon}{T}}\right)}\right).
\end{aligned}
\end{equation}
This completes the proof. $\hfill\square$
\bibliographystyle{IEEEtran}
\bibliography{reference}

\end{document}